\documentclass[a4paper,11pt]{article}
\usepackage[utf8]{inputenc}
\usepackage{cite}
 
\usepackage{graphicx}
\usepackage{graphicx}
\usepackage{pstricks}
\usepackage{appendix}
\usepackage{dsfont}
\usepackage{amsmath}
\usepackage{amssymb}
\usepackage{amsthm}
\usepackage{mathtools}
\usepackage{comment}
\usepackage[shortlabels]{enumitem}
\usepackage{tikz,lipsum,lmodern}
\usepackage[most]{tcolorbox}
\usepackage{bbm}

\usepackage{xcolor}
\usepackage{blindtext}

\usepackage[nottoc]{tocbibind}

\usepackage[linktoc=page, colorlinks, linkcolor=blue, citecolor=blue]{hyperref}
\usepackage[caption=false]{subfig} 
\usepackage{fixltx2e}
\usepackage{array}
\usepackage{verbatim}
\usepackage{bm}
\usepackage{verbatim}
\usepackage{textcomp}
\usepackage{mathrsfs}
\usepackage{epstopdf}
\usepackage{relsize}
\usepackage{cleveref} 
\usepackage{subfig}
 \usepackage{amsthm}
\usepackage{amssymb}

\newcommand{\nn}{\nonumber}

\newcommand{\calE}{\mathcal{E}}

\newcommand{\calN}{\mathcal{N}}

\newcommand{\calX}{\mathcal{X}}

\newcommand{\qednew}{\nobreak \ifvmode \relax \else
      \ifdim\lastskip<1.5em \hskip-\lastskip
      \hskip1.5em plus0em minus0.5em \fi \nobreak
      \vrule height0.75em width0.5em depth0.25em\fi}

\newcommand{\scrU}{\mathscr{U}}

\numberwithin{equation}{section}
\newtheorem{theorem}{Theorem}[section]

\newtheorem{proposition}[theorem]{Proposition}

\theoremstyle{remark}
\newtheorem{remark}[theorem]{Remark}

\DeclareMathOperator{\tr}{Tr}

\newcommand{\E}{{\mathbb E}}

\newcommand{\diag}{\operatorname{Diag}}

\newcommand{\sign}{\mathrm{sign}}

\title{Robust Mean Estimation under Quantization}

\begin{document}
\author{
  Pedro Abdalla\thanks{Department of Mathematics, University of California Irvine.}
  \and 
  Junren Chen\thanks{Department of Mathematics,   University of Maryland, College Park.}
}

\maketitle

\begin{abstract}
We consider the problem of mean estimation under quantization and adversarial corruption. We construct multivariate robust estimators that are optimal up to logarithmic factors in two different settings. The first is a one-bit setting, where each bit depends only on a single sample, and the second is a partial quantization setting, in which the estimator may use a small fraction of unquantized data.
\end{abstract}

\section{Introduction}
Parameter estimation under quantization is a fundamental problem at the intersection of statistics \cite{dirksen2022covariance}, signal processing \cite{dirksen2022sharp}, and machine learning \cite{cai2022distributed}. Quantization is of fundamental importance in these fields, mainly due to its role in reducing memory and storage costs.

In distributed learning, quantization plays a key role to reduce the cost of communication between data servers, where the central server has to estimate a parameter from the quantized data sent by the data servers. As pointed out in \cite{kipnis2022mean}, quantization also contributes to the recent paradigm of data privacy, since the quantized sample preserves sensitive information: the estimator only have access to bits, which reduces the chance to leak sensitive information.

In this work, we consider what is arguably the most fundamental parameter estimation task: the estimation of the mean of a random vector $X$ from $n$ i.i.d. copies $X_1,\ldots,X_n$ of $X$. Our goal is to estimate the mean using quantized samples $\dot{X}_1,\cdots,\dot{X}_n$, ideally using the smallest number of bits as possible. Following the terminology from \cite{kipnis2022mean}, we shall make an important distinction about the different settings that appear in the literature:

\begin{itemize}
    \item Centralized: the bits have access to the entire sample, that is, \\
    $\dot{X}_1(X_1,\ldots,X_n),\ldots, \dot{X}_n(X_1,\ldots,X_n)$.
    \item Adaptive: each bit is a function of the actual sample and the previous bits, that is, for every $j\in [n]$, $\dot{X}_j(X_j,\dot{X}_1,\ldots,\dot{X}_{j-1})$.
    \item Distributed:  each bit $\dot{X}_j$ depends only on the sample $X_j$, that is, \\ $\dot{X}_1(X_1),\ldots,\dot{X}_n(X_n)$ \footnote{Some works in the literature about the centralized setting refers to the problem as mean estimation in the distributed setting due to the terminology distributed learning.}.
\end{itemize}

It should be noted that the centralized setting is the least restrictive one. Indeed, in this setting the algorithm is allowed to manipulate unquantized data first and then quantize the output of the mean estimator. This makes the problem much easier than the other settings. On the other hand, the distributed setting is the most restrictive one as the samples are quantized directly without knowing anything about the other samples. Our work focuses  on the distributed setting but might be of interest in  the adaptive and centralized settings as well.

Our first main contribution is the construction of (near) optimal mean estimators for quantized data in the minimax sense. In particular, we show that our estimator can be easily extended to the high-dimensional setting, for which no results are currently available in the literature.

In the spirit of the active area of robust statistics \cite{lugosi2021trimmed,abdalla2024covariance,diakonikolas2023algorithmic}, we also show that our estimators are robust to
to an adversarial noise, meaning that it is accurate even when a small fraction of the data is compromised. To the best of our knowledge, our work is the first to study this practical relevant case of robust mean estimators under quantization. 

Our second main contribution is to propose a (slightly) more relaxed setting, where the algorithm is allowed to use a negligible fraction of unquantized data. Such setting which we shall refer to as the ``partial quantization setting" has the main advantage of avoiding the dependence on the location of the mean, and using almost the same number of bits. The results are reported in a progressive order of difficulty in Sections \ref{sec:fullquan} and \ref{sec:partialquan}, we refer the reader to Theorems \ref{thm:onebit_onedim} and \ref{thm:partial_multivariate_corruption} for the formal statements in the most complete form (multivariate with adversarial corruption).

\subsection*{Related Work}
The design of deterministic quantization maps for unbiased estimators is considered in 
\cite{venkitasubramaniam2007quantization,vempaty2014quantizer,chen2009performance} with focus on optimality with respect to the Fisher's information and Cramer-Rao lower bounds. However, they do not present sample complexity bounds for mean estimators. Our approach is significantly different, as the quantization map is random and has (negligible) bias which avoids the impossibility results from Cramer-Rao theory. Estimators such as the trimmed mean \cite{lugosi2021trimmed} or median of means \cite{lugosi2017lectures} have negligible bias, but they often present faster convergence rates than the optimal unbiased estimators.

Cai and Wei \cite{cai2022distributed,cai2024distributed} constructed minimax optimal estimators for the centralized setting. Their construction also holds for the multivariate case. On the other hand, the estimators are only available in the case that mean vector $\mu$ has all entries bounded by one and it is restricted to the case where $X$ follows a Gaussian distribution. In addition, their approach heavily relies on the fact that in the centralized setting one can compute an estimator for the mean and quantize the estimator's output, which is not suitable approach for the adaptive setting neither for the distributed setting. We remark that their method is very sensitive to bit flips, meaning that if one bit is corrupted then the whole procedure is fully compromised.

Kipnis and Duchi \cite{kipnis2022mean} proved sharp asymptotics with respect to the standard mean-square error (MSE) in the one-dimensional adaptive setting, assuming that the data is log-concave with known density. Kumar and Vatedka \cite{kumar2025one} improved the results of Kipnis and Duchi for the adaptive setting by relaxing the requirement on the knowledge of the variance. Recently,  near optimal non-asymptotic estimates for the adaptive setting were derived by Lau and Scarlett \cite{lau2025sequential} when the mean is bounded. All the approaches do not seem suitable to handle adversarial corruption due to sensitivity to bit flipping. 

Kipnis and Duchi also provided asymptotic lower bounds for the minimax MSE error in the distributed setting under some technical assumptions about the construction of the estimator. The authors claim that the convergence rate of any mean estimator in the distributed setting must depend on the location of the mean. The recent lower bound by Lau and Scarlett leads to the same conclusion. We provide a simple argument to prove this fact.

Suresh, Sun, Ro and Yu \cite{suresh2022correlated} constructed a quantization scheme for the distributed setting with only in-expectation guarantees, and depending on the empirical absolute deviation which is unclear how to control. Their method relies on a randomized quantization as ours but their construction require bounded data and it is limited to the one-dimensional setting.

The following results are somewhat related but do not overlap with our goal. Amiraz, Krauthgamer and Nadler \cite{amiraz2022distributed} considered the problem of estimating the support of $\mu$ assuming that $\mu$ is $K$-sparse with known $K$. Acharya, Canonne, Liu, Sun and Tyagi studied the density estimation \cite{acharya2021distributed} problem under quantization, restricted to discrete probability distributions only.

\subsection*{Organization}
The rest of this work is organized as follows. In Section \ref{sec:background}, we provide some background on robust statistics and dithering that is used in the construction of our estimators. In Section \ref{sec:fullquan}, we derive the main results for the distributed setting both in the noiseless case and in the case with adversarial corruption. In Section \ref{sec:partialquan}, we describe the partial quantization setting and derive the main results related to this setting. In Section \ref{sec:experiments}, we show numerical experiments to corroborate our theory.

\section{Background}
\label{sec:background}
In this section, we describe mathematically the mean estimation problem under the quantization constraint as well the necessary background. Given the sample that consists of $n$ i.i.d. copies $X_1,\ldots,X_n$ of a random vector $X$, our goal is to design a quantization scheme that maps the sample to bits $\dot{X}_1(X_1),\ldots,\dot{X}_n(X_n) \in \{0,1\}^d$, and an estimator $\widehat{\mu}(\dot{X}_1,\ldots,\dot{X}_n)$ for the mean of $X$, which we shall denote by $\mu \coloneqq\E X$.

\subsection{Background on robust statistics and dithering}

Our main idea is a novel connection between the trimmed mean estimator \cite{lugosi2021trimmed} in the robust statistics literature and the (random) dithering technique from the signal processing literature \cite{dirksen2021non,xu2020quantized,thrampoulidis2020generalized}.  

We start by describing the basics of the trimmed mean estimator. To simplify, we assume for now that $X$ is a one-dimensional random variable with mean $\mu$ and variance $\sigma^2$. The extension to the multivariate setting is postponed to the next section.

In a nutshell, the trimmed mean estimator cuts off samples that are too large in absolute value, and averages the remaining samples. That is, for some well-chosen parameters $\alpha$ and $\beta$, we define
\begin{align}
    \phi_{\alpha,\beta}(x) \coloneqq \begin{cases}
        \beta ,\qquad \text{if }x>\beta\\
        x,\qquad \text{if }x\in[\alpha,\beta]\\
        \alpha,\qquad \text{if }x<\alpha, 
    \end{cases}
\end{align}
and compute $$\frac{1}{n}\sum_{i=1}^n \phi_{\alpha,\beta}(X_i).$$

To describe how $\alpha$ and $\beta$ are chosen, we have to recall the notion of quantiles. To this end, we define the $p$-th quantile of $X$ by
\begin{equation}
\label{def:quantile}
Q_{p}(X) \coloneqq \sup\{ a\in \mathbb{R}: \mathbb{P}(X\ge a)\ge 1-p\}.
\end{equation}
Ideally, one would like to choose $\alpha = Q_{\varepsilon}(X)$ and $\beta = Q_{1-\varepsilon}(X)$ for  some small enough $\varepsilon$ to ensure that
\begin{equation*}
    \mathbb{P}(X\notin [\alpha,\beta]) = O(\sigma n^{-1/2}),
\end{equation*}
where the right-hand side is the minimax convergence rate of the (unquantized) mean estimation problem. This means that the bias introduced by the truncation function is negligible. 

Clearly, we do not have knowledge of the population quantiles. A natural approach is to replace by the empirical counterparts, that is, we consider a non-decreasing rearrangement of the data $X_1^{\ast} \le \cdots \le X_{N}^{\ast}$ and choose $\alpha = X_{\varepsilon}^{\ast}$ and $\beta=X_{1-\varepsilon}^{\ast}$. Intuitively, they should be close to $Q_{\varepsilon}(X)$ and $Q_{1-\varepsilon}(X)$, respectively. We postpone the technical details about the accuracy of this choice to the next section.

The dithering technique comes from the signal processing literature in one-bit estimation problems (see for example \cite{dirksen2023tuning,chen2025parameter,dirksen2021non,thrampoulidis2020generalized,dirksen2022covariance,chen2023high}). It is based on the following idea: for any (possibly deterministic) $x$ that is bounded $|x|\le \lambda$, we draw a random variable $\tau$ uniformly distributed over the interval $[-1, 1]$, namely $\tau \sim \scrU[-1,1]$. Next, the following formula is exploited, 
\begin{equation*}
    \E \big[\lambda\sign(x+\lambda\tau)\big] = x. 
\end{equation*}
The choice of $\tau$ guarantees that $\lambda\sign(x+\lambda\tau)$ is unbiased estimator for $x$. We adopt a somewhat different strategy. Inspired by the robust statistics literature, we consider the regime where the dithering may have a (negligible) bias. Indeed, notice that for any $x\in \mathbb{R}$
(not necessarily bounded) and $\lambda>0$, for a random variable $\tau \sim \scrU [-1,1]$,  we have that 
\begin{equation*}
\begin{split}
    &\E\Big[\lambda\sign(x+\lambda\tau)\Big] \\&= \lambda \mathbbm{1}(x>\lambda) -\lambda \mathbbm{1}(x<-\lambda) + \mathbbm{1}(|x|<\lambda)\Big[\lambda \mathbbm{P}(\tau>-x)-\lambda \mathbbm{P}(\tau<-x)\Big]\\
    &= \begin{cases}
       ~ \lambda,\,\qquad\text{if }x>\lambda\\
       ~ x,\qquad \,\text{if }|x|\le \lambda \\
        -\lambda, \qquad \text{if }x<-\lambda,
    \end{cases}
\end{split}
\end{equation*}
where the right-hand side is exactly the trimmed function, therefore
\begin{equation}
\label{eq:trimmed=dithering}
 \E\Big[\lambda\sign(x+\lambda\tau)\Big]=\phi_{-\lambda,\lambda}(x).
\end{equation}
The identity \eqref{eq:trimmed=dithering} actually shows that a biased dithering scheme may be beneficial to cut off outliers. This connection is a fundamental component in our proofs.

To end this section, we introduce the standard notion of adversarial noise that dates back to the work of Huber \cite{huber19721972} and now it is commonly adopted in the statistics and machine learning literature, for example in mean estimation or covariance estimation \cite{lugosi2021trimmed, abdalla2024covariance}. The idea behind adversarial noise is to use it as a quantitative benchmark to study the stability of estimators against imperfect data, which is relevant in practical applications. 

The mathematical description is relatively simple, we observe the data after $\eta n$ samples were arbitrarily changed. We shall make a distinction of two possibilities of adversarial noise in this manuscript.

We refer to pre-quantization noise when the adversary modifies the sample $X_1,\cdots,X_n$ to a corrupted sample $\widetilde{X}_1,\cdots,\widetilde{X}_n$ satisfying 
\begin{equation}
\label{eq:adversarial_noise_pre}
\bigg|\{i\in [n]: \widetilde{X}_i\neq X_i \}\bigg| \le \eta n.
\end{equation}
The noise affects the sample  collected before we apply the quantization procedure. Analogously, we refer to post-quantization noise when the adversarial corruption affects the samples after the quantization map is applied, that is, the quantization map outputs $\widetilde{\dot{X}}_1,\cdots,\widetilde{\dot{X}}_n$ satisfying
\begin{equation}
\label{eq:adversarial_noise_pos}
\bigg|\{i\in [n]: \widetilde{\dot{X}}_i \neq \dot{X}_i \}\bigg| \le \eta n.
\end{equation}

\section{Mean Estimation in the Distributed Setting}\label{sec:fullquan}
  In this section, we derive the main results for mean estimation for the distributed setting. However, we believe that our results are still relevant to the centralized or adaptive settings, specially due to the robustness properties of our estimators. 

To start, we make an important observation about the distributed setting. Unless we make additional assumptions on the distribution of $X$, the estimator $\widehat{\mu}$ is a function of the bits and nothing else. Consequently, there are only $2^n$ possible inputs to approximate any real value with arbitrarily high precision. For this reason, we need some extra knowledge about the data to construct our estimator. Later we show that such assumption is somewhat unavoidable. We first focus on the univariate case.

\subsection{The One-Dimensional Case}
For the one-dimensional case, we assume that a (rough) upper estimate $\lambda$ of the quantiles is available. More accurately, let $\overline{X}\coloneqq X-\E X$, given a confidence level $\delta>0$, we assume that
\begin{equation}
\label{eq:oracle_lambda}
\lambda \ge \max\big\{Q_{1-\varepsilon}(X),-Q_{\varepsilon}(X)\big\},
\end{equation}
where  $\varepsilon \asymp  \log(1/\delta)/n$. \footnote{In this work, the Landau's (asymptotic) notation only hides absolute constants.} Clearly, $\varepsilon$ must be less than one, thus we will assume for the rest of this text the regime $n=\Omega(\log(1/\delta))$. This entails no loss of generality, otherwise $\log(1/\delta)/n$ is bounded away from zero, and so is the error bound $|\widehat{\mu}-\mu|$ for any estimator $\widehat{\mu}$.

The upper bound requirement on $\lambda$ can be interpreted as a rough estimate of where most of the mass of $X$ is concentrated. Indeed, the data that lies outside of the interval $[Q_{\varepsilon}(X),Q_{1-\varepsilon}(X)]$ is irrelevant for mean estimation.

This weaker requirement on $\lambda$ is a major advantage compared to the previous assumptions in the literature that requires the mean to be bounded $|\mu|\le R$ for some known $R$ and the distribution of $X$ have a known log-concave density. In particular, our procedure applies to discrete distributions and it is not too sensitive to the exact knowledge about the properties of $X$. 

Next, we describe our quantization scheme. Given $X_1,\ldots,X_n$ i.i.d. copies of $X$ with mean $\mu$ an variance $\sigma^2$, we compute the quantized samples
\begin{align}
\label{eq:quantization_1dim_1bit}
\dot{X}_i := \sign(X_i+\lambda\tau_i),\qquad 1\le i\le n,
\end{align}
where $\tau_1,\cdots,\tau_n \sim \scrU[-1,1] $ are independent random variables. Our estimator is simply the empirical mean of the quantized samples,
\begin{align}
\label{eq:estimator_1dim_1bit}
\widehat{\mu}\big(\dot{X}_1,\cdots,\dot{X}_n)\coloneqq \frac{1}{n}\sum_{i=1}^n \lambda\dot{X}_i.
\end{align}
We are ready to state the first main result of our proposed scheme.
\begin{theorem}
\label{thm:onebit_onedim}
There exists an absolute constant $C>0$ for which the following holds. The estimator $\widehat{\mu}$ given by \eqref{eq:quantization_1dim_1bit} and \eqref{eq:estimator_1dim_1bit} satisfies with probability at least $1-\delta$,
\begin{equation*}
|\widehat{\mu}-\mu|\le C\max\big\{\sigma, \lambda\big\}\sqrt{\frac{\log(1/\delta)}{n}}.
\end{equation*}
Moreover, under the pre-quantization noise \eqref{eq:adversarial_noise_pre} or the post-quantization noise \eqref{eq:adversarial_noise_pos}, the estimator satisfies, with probability $1-\delta$, 
\begin{equation*}
|\widehat{\mu}-\mu|\le C\max\big\{\sigma, \lambda\big\}\sqrt{\frac{\log(1/\delta)}{n}} + 2\eta \lambda .
\end{equation*}
\end{theorem}
Before we proceed to the proof, we discuss the optimality of our result up to a multiplicative absolute constant. Notice that the term $\sigma\sqrt{\log(1/\delta)/n}$ is the minimax convergence rate for mean estimation even without quantization \cite{lugosi2017lectures}, thus it is optimal. In \cite{lugosi2021trimmed}, the authors proved that for any mean estimation procedure for bounded data, the term $\eta \|X\|_{L^\infty}$ is also necessary. This shows the optimality of the term $2\eta\lambda$ up to the constant factor $2$.

Next, we address the optimality with respect to $\lambda$ that is not present in the standard mean estimation problem. Notice that for any quantized sample $\lambda\dot{X} \in \{-\lambda,\lambda\}$, the variance of $\lambda \dot{X}$ is $\lambda^2 - (\E\dot{X})^2$ which is  $\Omega(\lambda^2)$ whenever $\E \lambda\dot{X} < \lambda/2$ (say). Thus, for large enough $n$ if the estimator has negligible bias we must have that the variance of $\dot{X}$ to be of order $\lambda^2$. The standard minimax lower bound for mean estimation applied to the i.i.d. quantized samples implies that $\lambda\sqrt{\log(1/\delta)/n}$ is also optimal.

One might ask how small $\lambda$ can be? It is not hard to see that it cannot be much smaller than $Q_{1-\varepsilon}(X)$. In fact, suppose that the law of $X$ is given by
\begin{align*}
    X \coloneqq \begin{cases}
        b ,\qquad \ \ \text{with prob } \varepsilon\\
        a,\qquad \ \ \text{with prob } (1-\varepsilon)/2\\
        -a,\qquad \text{with prob } (1-\varepsilon)/2.
    \end{cases}
\end{align*}
Notice that $Q_{1-\varepsilon}(X) = b$, thus for any $\lambda \ll b$, we obtain that an arbitrarily large gap in the bias $|\E \lambda\dot{X}-\mu|$. Therefore, we must assume that $\lambda$ is at least of the order of $Q_{1-\varepsilon}(X)$.  This shows that $\lambda$ must incorporate some prior knowledge about the location of the mean and the quantile of the centered distribution as $Q_{1-\varepsilon}(X)=Q_{1-\varepsilon}(\overline{X})+\mu$.

We conclude that our quantization scheme is optimal among all quantization schemes based on i.i.d. bits. Without assuming any prior knowledge on $X$, it is unlikely that a quantization scheme based on non-i.i.d. bits will improve the results.

We now proceed to the proof of Theorem \ref{thm:onebit_onedim}.

\begin{proof}
Recall the relation between dithering and the trimmed estimator in \eqref{eq:trimmed=dithering}. Next, since $\dot{X}_1,\ldots,\dot{X}_n$ are independent and bounded, Bernstein's inequality implies that with probability $1-\delta$,
\begin{equation}
\bigg|\frac{1}{n}\sum_{i=1}^n \lambda\dot{X}_i-\E \phi_{-\lambda,\lambda}(X)\bigg| \le C\lambda \sqrt{\frac{\log(1/\delta)}{n}}.
\end{equation}
All it remains is to control the bias. To this end, notice that 
\begin{equation*}
|\E \phi_{-\lambda,\lambda}(X) - \mu|\le |\mathbb{E}(X-\lambda)\mathds{1}_{X>\lambda}| + |\mathbb{E}(-\lambda-X)\mathds{1}_{X<-\lambda}|.
\end{equation*}
We handle the first term on the right-hand side, the second follows analogously. Notice that by Cauchy-Schwarz inequality
\begin{equation*}
\begin{split}
&|\mathbb{E}(X-\lambda)\mathds{1}_{X>\lambda}| \le |\mathbb{E}(X-\E X)\mathds{1}_{X>\lambda}| + |\E X-\lambda|\mathbb{P}(X\ge \lambda)\\
&\le \sigma \sqrt{\mathbb{P}(X\ge \lambda)} +|\E X-\lambda| \mathbb{P}(X\ge \lambda)\\
&=O\left( \sigma\sqrt{\frac{\log(1/\delta)}{n}}+ |\lambda-\mu|\mathbb{P}(X-\mu\ge \lambda-\mu)\right).
\end{split}
\end{equation*}
To control the second term $|\lambda-\mu|\mathbb{P}(X-\mu\ge \lambda-\mu)$, we apply Chebyshev's inequality to obtain that
\begin{equation*}
    \mathbb{P}(X-\mu\ge \lambda-\mu)\le \frac{\sigma^2}{|\lambda-\mu|^2},
\end{equation*}
thus $|\lambda-\mu|\le \sigma/\sqrt{\mathbb{P}(X-\mu\ge \lambda-\mu)}$ or equivalently,

$$|\lambda-\mu|\mathbb{P}(X-\mu\ge \lambda-\mu)\le \sigma \sqrt{\mathbb{P}(X-\mu\ge \lambda-\mu)}.$$

\noindent Recalling that $Q_{1-\varepsilon}(X)=Q_{1-\varepsilon}(\overline{X})+\mu$, we obtain that $|\lambda-\mu|\ge Q_{1-\varepsilon}(\overline{X})$ by \eqref{eq:oracle_lambda}, thus we conclude 
\begin{equation*}
|\lambda-\mu|\mathbb{P}(X-\mu\ge \lambda-\mu) \le \sigma \sqrt{\varepsilon}=O\left(\sigma\sqrt{\frac{\log(1/\delta)}{n}} \right),
\end{equation*}
where in the last step we used our choice of $\varepsilon \asymp \log(1/\delta)/n$. We now prove robustness. Since each $\dot{X}_i$ only depends on $X_i$, the pre-quantization noise and post-quantization noise are equivalent. Therefore, for any sequence $\widetilde{\dot{X}}_1,\cdots,\widetilde{\dot{X}}_n$ satisfying \eqref{eq:adversarial_noise_pos}, we have  
\begin{equation*}
\bigg|\frac{1}{n}\sum_{i=1}^n \lambda\dot{X}_i - \frac{1}{n}\sum_{i=1}^n \lambda\widetilde{\dot{X}}_i \bigg| =\lambda\bigg|\frac{1}{n}\sum_{i=1}^n \dot{X}_i - \frac{1}{n}\sum_{i=1}^n \widetilde{\dot{X}}_i \bigg| \le 2\eta \lambda,
\end{equation*}
where we used the fact that $|\dot{X}_i|\le 1$, for any $i\in [n]$.
\end{proof}

\subsection{The Multivariate Case}
In this section, we extend our procedure to the multivariate setting. Now, $X$ denotes a random vector in $\mathbb{R}^d$ with mean $\mu$ and  covariance matrix is $\Sigma \coloneqq \E (X-\mu) (X-\mu)^T$. To describe our quantization map $X\rightarrow \dot{X}$, let $\Lambda  \coloneqq (\lambda_1,\cdots,\lambda_d)$ be a vector such that
\begin{equation}\label{multituning}
\lambda_i\ge \max\big\{Q_{1-\varepsilon}\big(\langle X,e_i\rangle\big),-Q_{\varepsilon}\big(\langle X,e_i\rangle\big)\big\}, \quad  i=1,\cdots d.
\end{equation}
Here we also assume the regime $\varepsilon \asymp \log(1/\delta)/n$. Next, we generate the sequence $\tau_1,\cdots,\tau_d$ by sampling (independently) $\tau_i \sim \scrU [-1,1] $ and quantize the samples using the following expression 
\begin{align*}
\dot{X} := \begin{bmatrix}
\sign\big(\langle X,e_1\rangle + \lambda_1\tau_{1}\big) \\
\sign\big(\langle X,e_2\rangle + \lambda_2\tau_{2}\big) \\
\vdots \\
\sign\big(\langle X,e_d\rangle + \lambda_d\tau_{d}\big)
\end{bmatrix}.
\end{align*}
To short notation, we shall write 
\begin{align}
\label{eq:quantization_ddim_1bit}
\dot{X}\coloneqq \sign (X + \diag(\Lambda)\tau),
\end{align}
where $\diag(\Lambda)$ is the diagonal matrix whose entries are equal to the entries of the vector $\Lambda$ and $\tau \coloneqq (\tau_1,\cdots,\tau_d) \sim \scrU [-1,1]^d$. Here the sign function $\sign(\cdot)$ is taken entrywise.

Our first proposed estimator is the empirical mean of the quantized sample $\dot{X}_1,\cdots,\dot{X}_n$, namely 
\begin{align}
\label{eq:estimator_ddim_1bit}
\widehat{\mu}\big(\dot{X}_1,\cdots,\dot{X}_n)\coloneqq \frac{1}{n}\sum_{i=1}^n \diag(\Delta)\dot{X}_i.
\end{align}

\begin{theorem}[Multivariate without corruption]
\label{thm:onebit_ddim_without_robust}
There exists an absolute constant $C>0$ for which the following holds. The estimator $\widehat{\mu}$ given by \eqref{eq:quantization_ddim_1bit} and \eqref{eq:estimator_ddim_1bit} satisfies with probability at least $1-\delta$,
\begin{equation*}
\|\widehat{\mu}-\mu\|_2\le C\left( \max_{i\in [d]}\lambda_i\sqrt{\frac{\log(1/\delta)}{n}}+\sqrt{\frac{\sum_{i=1}^d\lambda_i^2}{n}}+\sqrt{\tr(\Sigma)\frac{\log(1/\delta)}{n}}\right).
\end{equation*}
\end{theorem}
We remark that the first two terms on the right-hand side are optimal up to an absolute constant. First, any multivariate mean estimator $\widehat{\mu}$ for the mean $\mu$ is also an estimator for $\langle \mu,e_i\rangle$ (for any $i\in [d]$), thus the first term is unavoidable. Second, the covariance matrix of the $\dot{X}$ has trace equal to $\sum_{i=1}^d \lambda_i^2$ regardless of the quantization map. Thus, the second term on the right-hand side is minimax optimal among all estimator based on i.i.d. bits by the standard minimax rate for multivariate mean estimation \cite{lugosi2017lectures}. Finally, the third term is the standard minimax rate for multivariate mean estimation \cite{lugosi2017lectures} up to an extra log factor $\log(1/\delta)$.

As a technical ingredient in our proof, we need to recall the notion of subgaussian random vector. For a more comprehensive introduction to the topic, we refer the reader to \cite{vershynin2018high}. We say that a random variable $Z\in \mathbb{R}$ is $K$-subgaussian if there exists a $K>0$ such that
\begin{equation*}
   \mathbb{P}\bigg(|Z-\E Z|\ge (K\sigma_Z)t\bigg)\le 2e^{-t^2/2}, \quad \text{for all $t>0$.}
\end{equation*}
The smallest $K>0$ for which this estimate holds is denoted by $\|Z\|_{\psi_2}$($\psi_2$ norm). Similarly, we say that a random vector $X$ is $K$-subgaussian if there exists an $K>0$ such that 
\begin{equation}
\label{eq:def_subgaussian}
    \|\langle X,v\rangle\|_{\psi_2}\le K \sqrt{\langle \Sigma v,v\rangle}, \quad \text{for every $v\in \mathbb{R}^d$.}
\end{equation}

\begin{proof}
We start by decomposing the error estimate into two parts, a bias term and a concentration term. By triangle inequality,
\begin{equation*}
\|\widehat{\mu}-\mu\|_2\le \|\widehat{\mu}-\E \widehat{\mu}\|_2 + \|\mu - \E \widehat{\mu}\|_2.
\end{equation*}
We first quantify how close our estimator is to its expected value (concentration term). To this end, notice that the covariance matrix $\dot{\Sigma}$ of the random vector $\diag(\Delta)\dot{X}$ is a diagonal matrix whose values are $\lambda_1^2,\ldots,\lambda_d^2$. Moreover, conditionally on $X$, $\dot{X}$ has independent entries. Thus by \cite[Proposition 2.7.1]{vershynin2018high} 
\begin{equation*}
    \|\langle \dot{X},v\rangle\|_{\psi_2}^2\le\sum_{j=1}^d \|\langle \dot{X},e_j\rangle\|_{\psi_2}^2v_j^2\le c \sum_{j=1}^d\lambda_j^2 v_j^2 = c\langle \dot{\Sigma}v,v\rangle,
\end{equation*}
for some absolute constant $c>0$. Consequently, the standard estimate for the empirical mean $\widehat{\mu}$ (see for example \cite[Equation 1.5]{lugosi2019near}) satisfies with probability $1-\delta$,
\begin{equation*}
\big\|\widehat{\mu}-\E \widehat{\mu}\big\|_2 = O\left(\sqrt{\frac{\sum_{i=1}^d\lambda_i^2}{n}} + \max_{i\in [d]}\lambda_i\sqrt{\frac{\log(1/\delta)}{n}}\right).
\end{equation*}
Now, it suffices to control the bias. 
To this end, notice that 
\begin{equation*}
    \mu - \E \widehat{\mu} = \mu-\E[\diag(\Delta)\dot{X}]=
\begin{bmatrix}
\ \E\big(\langle X,e_1\rangle -\lambda_1\big)\mathds{1}_{|\langle X,e_1\rangle|\ge \lambda_1}\\
\vdots \\
\E\big(\langle X,e_d\rangle -\lambda_d\big)\mathds{1}_{|\langle X,e_d\rangle|\ge \lambda_d}
\end{bmatrix},
\end{equation*}
which implies that we need to control the quantity
\begin{equation*}
\bigg|\langle \mu-\E \widehat{\mu},v\rangle\bigg|\le \bigg|\sum_{j=1}^d \E\langle X,e_j\rangle-\lambda_j\mathds{1}_{\langle X,e_j\rangle\ge \lambda_j}v_j\bigg| + \bigg|\sum_{j=1}^d \E\langle X,e_j\rangle+\lambda_j\mathds{1}_{\langle X,e_j\rangle\le -\lambda_j}v_j\bigg|.
\end{equation*}
We focus on the first term on the right-hand side, while the second admits the same analysis. By triangle inequality and
Cauchy-Schwarz inequality, we have for every $v\in S^{d-1}$
\begin{equation*}
\begin{split}
&\bigg|\sum_{j=1}^d \E\langle X,e_j\rangle-\lambda_j\mathds{1}_{\langle X,e_j\rangle\ge \lambda_j}v_j\bigg|\\
&\le \bigg|\sum_{j=1}^d \E\overline{\langle X,e_j\rangle}\mathds{1}_{\langle X,e_j\rangle\ge \lambda_j}v_j\bigg| + \bigg|\sum_{j=1}^d \E(\lambda_j-\langle \mu,e_j\rangle)\mathds{1}_{\langle X,e_j\rangle\ge \lambda_j}v_j\bigg| \\
&\le \sqrt{\sum_{j=1}^d(\E\overline{\langle X,e_j\rangle}\mathds{1}_{\langle X,e_j\rangle\ge \lambda_j})^2} + \sqrt{\sum_{j=1}^d(\lambda_j-\langle \mu,e_j\rangle)^2\mathbb{P}\big(\langle X,e_j\rangle\ge \lambda_j \big)^2}\\
&\le \sqrt{\sum_{j=1}^d \Sigma_{jj} \mathbb{P}(\langle X,e_j\rangle\ge \lambda_j)} + \sqrt{\sum_{j=1}^d(\lambda_j-\langle \mu,e_j\rangle)^2\mathbb{P}\big(\langle X,e_j\rangle\ge \lambda_j \big)^2}\\
&\le \sqrt{\tr(\Sigma)\frac{\log(1/\delta)}{n}} + \sqrt{\sum_{j=1}^d(\lambda_j-\langle \mu,e_j\rangle)^2\mathbb{P}\big(\overline{\langle X,e_j\rangle}\ge \lambda_j-\langle \mu,e_j\rangle \big)^2}\\
&=O\left(\sqrt{\tr(\Sigma)\frac{\log(1/\delta)}{n}}\right),
\end{split}
\end{equation*}
where in the last step, we estimated the second term on the right-hand side using that $$\big(\lambda_j-\langle \mu,e_j\rangle\big)\mathbb{P}\big(\overline{\langle X,e_j\rangle}\ge \lambda_j-\langle \mu,e_j\rangle)\le \sqrt{\Sigma_{jj}\log(1/\delta)/n}. $$ Thus, we take the supremum over all $v\in S^{d-1}$ on both sides to conclude that
\begin{equation*}
\|\mu-\E \widehat{\mu}\|_2 =\sup_{v\in S^{d-1}} |\langle \mu-\E \widehat{\mu},v\rangle|=O\left(\sqrt{\frac{\tr(\Sigma)}{n}\log(1/\delta)}\right).
\end{equation*}
\end{proof}

\begin{remark}
In view of Equation (\ref{multituning}), one needs to deal with $d$ tuning parameters $\{\lambda_i\}_{i=1}^d$. For each $i\in [d]$, one can decompose $Q_{1-\varepsilon}(\langle X,e_i\rangle)= Q_{1-\varepsilon}(\overline{\langle X,e_i\rangle})+\langle \mu,e_i\rangle$ as the sum of rough estimate of centered quantile and the location of the $i$-th coordinate of the mean.

When one has less knowledge on the mean vector $\mu$ than on the centered distribution, meaning that no rough bound on $\langle\mu, e_i\rangle$ is available, we advocate the following Haar randomization step: we first draw a Haar matrix $O\sim {\rm Haar}(d)$ \footnote{This means that $O$ is uniformly distributed over the orthogonal group of all orthonormal matrices in $\mathbb{R}^{d\times d}$.} and pre-process $\{X_i\}_{i=1}^n$ to $\{X_{i,O}:=OX_i\}_{i=1}^n$, then implement the same (quantization and estimation) procedure as in Theorem \ref{thm:onebit_ddim_without_robust} with $\{X_{i,O}\}_{i=1}^n$ to obtain the estimate $\hat{\mu}$, and finally use $O^T\hat{\mu}$ as the final estimate. This might alleviate the tuning issue since one now only needs to have a crude (upper) estimate on $\|\mu\|_2$ instead of $\{\langle \mu,e_i\rangle\}_{i=1}^d$: to see this, note that
the mean of $X_{i,O}$ is given by $\mu_{O}:=O\mu$ and thus uniformly distributed over $\|\mu\|_2S^{d-1}$. Consequently, a standard union bound argument shows that \[\max_{1\le i\le d}\,|\langle O\mu,e_i\rangle| \le\|\mu\|_2\sqrt{\frac{3\log d}{d}}\] with high probability, which allows us to set 
\[\lambda_i\ge \max\bigg\{Q_{1-\varepsilon}(\langle O\overline{X},e_i\rangle)+\|\mu\|_2\sqrt{\frac{3\log d}{d}},-Q_{\varepsilon}(\langle O\overline{X},e_i\rangle)-\|\mu\|_2\sqrt{\frac{3\log d}{d}}\bigg\}\]
to estimate mean $O\mu$ using the samples $\{X_{i,O}\}_{i=1}^n$. This argument might be of practical interest to distributed learning settings, as the generation and communication of $O$ can be handled via random seed and it does not cost much (see, e.g., \cite{shrivastava2024sketching}) and the convergence rate is only affected by an extra $O(\sqrt{\log d})$ factor. 
\end{remark}

In contrast to the one-dimensional case, the empirical mean of the quantized samples does not seem to have strong guarantees against the adversarial noise. Following the proof above, one can show that the additional term due to corruption is of order $O\bigg(\eta\sqrt{\sum_{i=1}^d\lambda_i^2}\bigg)$, and it is unlikely to admit substantial improvements without changing the estimator. We will provide a numerical example to corroborate this in Section \ref{sec:experiments}. 

Based on this fact, we exploit the literature in robust statistics and recall a remarkable result due to Depersin and Lecu{\'e} in robust mean estimation, which we shall state using our notation.

\begin{proposition}[Depersin and Lecu{\'e}\cite{depersin2022robust}]
\label{prop:depersin_lecue}
There exists a constant $C>0$ for which the following holds. Assume the regime where $\log(1/\delta)\ge 300 \eta n$. There exists a polynomial-time algorithm in the sample size $n$ and dimension $d$ that takes as an input $n$ corrupted samples of a random vector $X$ whose mean $\mu$ and covariance $\Sigma$ satisfying \eqref{eq:adversarial_noise_pre} and outputs a vector $\widehat{\mu}$ satisfying with probability at least $1-\delta$,
\begin{equation*}
\|\widehat{\mu}-\mu\|_2\le C\left(\|\Sigma\|\sqrt{\frac{\log(1/\delta)}{n}+\eta} + \sqrt{\frac{\tr(\Sigma)}{n}}\right).
\end{equation*}
\end{proposition}
Next, we improve the robustness guarantee of our quantization scheme with a more involved estimator. 
\begin{theorem}[Multivariate with corruption]
\label{thm:onebit_ddim_with_robust}
There exists a constant $C>0$ for which the following holds. Let $\widehat{\mu}$ be the estimator from Proposition \ref{prop:depersin_lecue} with the input being $\eta n$ corrupted samples of $\dot{X}$ given by (\ref{eq:quantization_ddim_1bit}) satisfying \eqref{eq:adversarial_noise_pre} or \eqref{eq:adversarial_noise_pos}. Then, with probability at least $1-\delta$, 
\begin{equation*}
\|\widehat{\mu}-\mu\|_2\le C\left( \max_{i\in [d]}\lambda_i\sqrt{\frac{\log(1/\delta)}{n}+\eta}+\sqrt{\frac{\sum_{i=1}^d\lambda_i^2}{n}}+\sqrt{\tr(\Sigma)\frac{\log(1/\delta)}{n}}\right).
\end{equation*}
\end{theorem}
\begin{proof}
We apply the same decomposition as before $\|\widehat{\mu}-\mu\|_2\le \| \mu-\E \widehat{\mu}\|_2 + \| \widehat{\mu}-\E \widehat{\mu}\|_2 $. The estimate for the bias term $\| \mu-\E \dot{X}\|_2$ follows exactly the same steps as in the proof of Theorem \ref{thm:onebit_ddim_without_robust}. Next, recall that $\diag(\Delta)\dot{X}$ is a random vector with a diagonal covariance matrix equal to $\diag(\lambda_1,\cdots,\lambda_d)$. Consequently, the operator norm of the covariance matrix is $\max_{i\in [d]}\lambda_i$ and its trace is equal to $\sum_{i=1}^d \lambda_i^2$. We apply the estimator from Proposition \ref{prop:depersin_lecue} to the sample $\diag(\Delta)\dot{X}_1,\cdots,\diag(\Delta)\dot{X}_n$ to conclude the proof.
\end{proof}
Although the factor $\max_{i\in [d]}\lambda_i\sqrt{\eta}$ is sharp when $X$ only has two finite moments \cite{lugosi2021trimmed}, it may be improved under stronger assumptions on the tail of $X$. However, we do not know how to achieve that using a polynomial-time algorithm.

We also remark that it might be the case that the extra log-factor could be improved using a version of the multivariate trimmed mean estimator from \cite{lugosi2021trimmed}, however such estimator is known to be computationally intractable (running time is exponential with respect to the dimension).

\section{Mean Estimation under Partial Quantization}\label{sec:partialquan}

In Section \ref{sec:fullquan}, we derived (near) optimal estimators for the mean of a random vector under the quantization constraint in the distributed setting. As observed, any estimator requires extra knowledge of the data that is sensitive to the location of the mean. In practice, the dependence on extra knowledge of the data also forces the algorithm to require tuning parameters. One may ask if it is possible to relax a bit the distributed setting to avoid dependence on the location of the mean in the error rates and alleviate the tuning issue.

While fully quantized communication models are widely studied in the literature, they can be overly restrictive in certain regimes. In practice, it might be of interest to relax the communication constraints to not be uniform: while most observations must be heavily quantized due to bandwidth or energy limitations, a small number of real-valued messages can often be transmitted for coordination or refinement purposes. In fact, a real number is modeled as a fixed-precision floating-point number, say, single precision with 32 bits.

The goal of this section is to relax the restrictive distributed setting by allowing the use of a $n_0=o(n)$ unquantized data before quantizing the remaining $n-n_0$ samples. The amount of bits used in this setting is $32n_0+(n-n_0)=(1+o(1))n$, which is only slightly larger than the number of bits used in the distributed setting. We shall refer to it as the ``partial quantization'' setting. 

On the other hand, as we shall see, our estimator becomes completely tuning free and the convergence rates no longer depend on the location of the mean.

\subsection{The One-Dimensional Case}
We shall start with the univariate case without corruption. In what follows, we denote the sample by $Y_1,Y_2,\cdots,Y_{n_0},X_{n_0+1},\cdots,X_n$ which are i.i.d. copies of the random variable $X$ whose mean is $\mu$ and variance $\sigma^2$. We consider $n_0$ unquantized samples $\calX_{\rm (f)}\coloneqq\{Y_1,\cdots,Y_{n_0}\}$ and quantize the remaining samples $\{X_i\}_{i=n_0+1}^n$,
possibly using information from the samples in  $\calX_{\rm (f)}$.

Our goal is to estimate the mean $\mu$ based on the unquantized samples from $\calX_{\rm (f)}$ and the samples $\{X_i\}_{i=n_0+1}^n$ after quantization. Note that we seek for a mean estimator that is accurate in the regime where $n_0=o(n)$, which leads to a reduction of bit number from $32n$ to $(n-n_0)+32n_0=(1+o(1))n$ which is only slightly larger than the one-bit per sample in the distributed setting considered in the previous section.

Our estimation procedure is also inspired by the trimmed mean estimator described in Section \ref{sec:background}. In this context, our estimator uses the data from $\calX_{\rm (f)}$ as empirical quantiles of $X$ to cut off outliers. Specifically, for some well-chosen parameter $\varepsilon$, we propose to compute the empirical $\varepsilon$-quantile and $(1-\varepsilon)$-quantile of $\calX_{\rm (f)}$ 
  \begin{align}
      \alpha:=Y^*_{\varepsilon n_0}\qquad{\rm and}\qquad \beta:= Y^*_{(1-\varepsilon)n_0}, 
  \end{align}
  from which we further compute a crude estimate $\mu_1$ of the mean $\mu$ and an approximate range of the unquantized samples given by 
  \begin{align}
  \label{def:mu_1}
     \mu_1:=\frac{\alpha+\beta}{2} \qquad{\rm and}\qquad\Delta:=\frac{\beta-\alpha}{2}.
  \end{align}
Next, to quantize the samples $\{X_i\}_{i=n_0+1}^n$, we propose to shift the data by $\mu_1$ and then employ independent uniform dithers $\tau_i$ with dithering level $\Delta$. That is, we draw an i.i.d. sequence of random variables $\{\tau_i\}_{i=n_0+1}^n$ distributed as ${\scrU}([-1,1])$ and then retain the following $n-n_0$ bits 
\begin{align}
\label{eq:quantization_partial_onedim}
      \dot{X}_i := \sign(X_i-\mu_1+\Delta\tau_i),\qquad n_0+1\le i\le n.  
\end{align}
We exploit again the relationship between the trimmed estimator and dithering  (\ref{eq:trimmed=dithering})  to derive that
\begin{align}
\label{expdither}
    \mathbb{E}(\Delta \dot{X}|X) = \phi_{\alpha,\beta}(X) - \mu_1.
\end{align}
Our estimator is given by
\begin{align}
\label{hatmu}
    \hat{\mu} \coloneqq \frac{\Delta}{n-n_0}\sum_{i=n_0+1}^n \dot{X}_i + \mu_1. 
\end{align}
To state our results accurately, we define for $t\in (0,1)$
\begin{equation}
\label{eq:error_term_real_corfree}
\begin{split}
&\calE(\overline{X};t):=\max\Big\{\mathbb{E}\Big[|\overline{X}|\mathbf{1}\{\overline{X}_i\le Q_{t}(\overline{X})\}\Big],\mathbb{E}\Big[|\overline{X}|\mathbf{1}\{\overline{X}\ge Q_{1-t}(\overline{X})\}\Big]\Big\}
\end{split}.
\end{equation}
To simplify, we state our results with some specified probability guarantee rather than the general confidence level $\delta$ as before. This will simplify the expressions and avoid convoluted statements.

For the reader's convenience, we recall our notation for the centered version of $X$, namely $\overline{X}=X-\E X$. We now state our result for the simplest case in the partial quantization setting.

\begin{theorem} \label{thm:real_corfree}
There exist absolute constants $C,c_1,c_2>0$ for which the following holds. Consider the  quantization map in \Cref{eq:quantization_partial_onedim} and the estimator in \Cref{hatmu} with parameters $n_0\le \frac{n}{2}$ and $\varepsilon>0$. Then with probability at least $1-4\exp(-c_1n_0\varepsilon)-\frac{2}{n}$, we have 
\begin{align}
\begin{split}
|\hat{\mu}-\mu|&\le \mathcal{E}\Big(\overline{X};\frac{3\varepsilon}{2}\Big)+3\varepsilon\big|\max\big\{-Q_{1-\varepsilon/2}(\overline{X}),Q_{3\varepsilon/2}(\overline{X})\big\}\big|\\
&+\bigg(Q_{1-\varepsilon/2}(\overline{X})-Q_{\varepsilon/2}(\overline{X})\bigg)\sqrt{\frac{\log n}{n}}. \label{genebou11}
\end{split}
\end{align}
In particular, if $X$ is  $K$-subgaussian then under the regime $n_0:=\sqrt{n }\log n$ and $\varepsilon:=c_2\sqrt{\frac{\log n}{n}}$, with probability at least $1-\frac{1}{n}$, we have 
\begin{align}
    |\hat{\mu}-\mu|\le CK\frac{\log n}{\sqrt{n}}.
\end{align}
\end{theorem}
Notice that the convergence rate depends on the quantiles of $\overline{X}$ rather than the quantiles of $X$. Consequently, the dependence on the location of the mean is completely avoided. In addition, for subgaussian $X$, the convergence rate is optimal up to a mild log-factor $\sqrt{\log n}$, which shows that our results are quite close to the minimax optimal rate. Finally, we remark that the tail probability estimates $n^{-1}$ could be replace by a faster decay $n^{c}$ for any constant $c>1$ by just changing the absolute constants in the theorem.

\begin{proof}
Recall the definition of quantiles in \Cref{def:quantile} and the fact that $Q_\varepsilon(X)=Q_\varepsilon(\overline{X})+\mu$. Notice that for any $0<\delta< \varepsilon$,
\begin{equation*}
|\{i\in[n_0]:Y_i\le Q_{\varepsilon-\delta}(X)\}|\sim {\rm Bin}(n_0,\varepsilon-\delta).
\end{equation*}
Thus, by Chernoff's inequality
\begin{align}
    \mathbb{P}\left(|\{i\in[n_0]:Y_i\le Q_{\varepsilon-\delta}(X)\}| \le \bigg(\varepsilon-\frac{\delta}{2}\bigg)n_0\right)\ge 1- p_1,
\end{align}
where $p_1:=\exp\Big(-n_0 {\rm D_{KL}}\Big(\varepsilon-\frac{\delta}{2}\big\|\varepsilon-\delta\Big)\Big)$. 
Similarly, we have 
\begin{align}
     \mathbb{P}\left(|\{i\in[n_0]:Y_i\ge Q_{\varepsilon+\delta}(X)\}| \le \bigg(1-\varepsilon-\frac{\delta}{2}\bigg)n_0\right)\ge 1-p_2 
\end{align}
where $p_2:=\exp\Big(-n_0{\rm D_{KL}}\Big(\varepsilon+\frac{\delta}{2}\big\|\varepsilon+\delta\Big)\Big)$. In particular, setting $\delta\coloneqq \varepsilon/2$, there exists an absolute constant $c_1>0$ for which it holds that
\begin{equation*}
p_1+p_2\le  e^{-c_1n_0\varepsilon} +e^{-c_1n_0\varepsilon} = 2e^{-c_1n_0\varepsilon},
\end{equation*}
and consequently
\begin{align}\label{alphabound}
    \mathbb{P}\Big(Q_{\varepsilon/2}(X)\le \alpha \le Q_{3\varepsilon/2}(X)\Big)\ge 1-(p_1+p_2)\ge 1-2e^{-c_1n_0\varepsilon}.
\end{align}
Analogously, we have 
\begin{align}\label{betabound}
    \mathbb{P}\Big(Q_{1-\varepsilon/2}(X)\le\beta\le Q_{1-3\varepsilon/2}(X)\Big)\ge 1-2e^{-c_1n_0\varepsilon}.
\end{align} 
For the rest of the proof, we condition on the event where \eqref{alphabound} and \eqref{betabound} hold. Next, we combine (\ref{expdither}) and (\ref{hatmu}) and take the expectation with respect to the dithers to obtain
\begin{align}
    \mathbb{E}(\hat{\mu}) = \mathbb{E}\big[\phi_{\alpha,\beta}(X)\big]. \label{exptau}
\end{align}
By the definition of the truncation function $\phi_{\alpha,\beta}(\cdot)$, \eqref{exptau} combined with \eqref{alphabound} and \eqref{betabound} implies that 
\begin{align}\nn
    &\mathbb{E}(\widehat{\mu})\le \mathbb{E}\Big[\phi_{Q_{3\varepsilon/2}(X),Q_{1-\varepsilon/2}(X)}(X)\Big] \\\nn
    &= \mu+ \mathbb{E}\Big[\phi_{Q_{3\varepsilon/2}(X),Q_{1-\varepsilon/2}(X)}(X)-X\Big]\\
    &\le \mu + \mathbb{E}\Big[\big(Q_{3\varepsilon/2}(X)-X\big)\mathbf{1}(X\le Q_{3\varepsilon/2}(X))\Big]\label{4.122}\\\nn
    &= \mu + \mathbb{E}\Big[\big(Q_{3\varepsilon/2}(\overline{X})-\overline{X}\big)\mathbf{1}(\overline{X}\le Q_{3\varepsilon/2}(\overline{X}))\Big]\\\nn
    &\le \mu + \mathbb{E}\Big[\big|\overline{X}\big|\mathbf{1}(\overline{X}\le Q_{3\varepsilon/2}(\overline{X}))\Big] + Q_{3\varepsilon/2}(\overline{X})\frac{3\varepsilon}{2}.
\end{align}
Similarly, we have 
\begin{align}\nn
    &\mathbb{E}(\hat{\mu}) \ge \mu - \mathbb{E}\Big[|\overline{X}|\mathbf{1}(\overline{X}\ge Q_{1-\varepsilon/2}(\overline{X}))\Big] + Q_{1-\varepsilon/2}(\overline{X})\frac{3\varepsilon}{2}. 
\end{align}
Recalling the error term in \Cref{eq:error_term_real_corfree}, we conclude that the bias term is controlled by
\begin{align}
|\mathbb{E}(\hat{\mu})-\mu|&\le\mathcal{E}(\overline{X};3\varepsilon/2) + 3\varepsilon\big|\max\big\{-Q_{1-\varepsilon/2}(\overline{X}),Q_{3\varepsilon/2}(\overline{X})\big\}\big|, \label{deviation}
\end{align} 
which corresponds to the first two terms in the right-hand side of \Cref{genebou11}. It remains to control the concentration term $|\hat{\mu}-\mathbb{E}(\hat{\mu})|$. To this end, we write  
\begin{align}
    \big|\hat{\mu}-\mathbb{E}(\hat{\mu})\big| =\left|\frac{1}{n-n_0}\sum_{i=n_0+1}^n \big(\Delta \dot{X}_i - \mathbb{E}\big[\Delta \dot{X}_i\big]\big)\right|,
\end{align}
and apply Hoeffding's inequality to obtain for any $t>0$
\begin{align}
    \mathbb{P}\Big(|\hat{\mu}-\mathbb{E}(\hat{\mu})|\ge t\Big)\le 2\exp\Big(-\frac{(n-n_0)t^2}{2\Delta^2}\Big).
\end{align}
In particular, setting $t = \Delta \sqrt{\frac{2\log n}{n-n_0}}$ implies that with probability at least $1-\frac{2}{n}$, we have
\begin{align}\label{concen}
    |\hat{\mu}-\mathbb{E}(\hat{\mu})|\le \Delta\sqrt{\frac{2\log n}{n-n_0}}\le 2\Delta\sqrt{\frac{\log n}{n}}.
\end{align}
On the events (\ref{alphabound}) and (\ref{betabound}), we have 
\begin{align}\label{Detabound}
    \Delta = \frac{\beta -\alpha}{2}\le \frac{Q_{1-3\varepsilon/2}(\overline{X})-Q_{\varepsilon/2}(\overline{X})}{2}.
\end{align}
Thus, from \Cref{deviation} and \Cref{concen}
\begin{align*}
    |\widehat{\mu}-\E\widehat{\mu}|\le \big(Q_{1-3\varepsilon/2}(\overline{X})-Q_{\varepsilon/2}(\overline{X})\big)\sqrt{\frac{\log n}{n}},
\end{align*}
with probability at least $1-4e^{-c_1n_0\varepsilon}-\frac{2}{n}$ which concludes the proof of the first part. For the second part of the statement, we apply standard tail estimates for subgaussian distributions to obtain that
\begin{gather}
  \varepsilon\big|\max\big\{-Q_{1-\varepsilon/2}(\overline{X}),Q_{3\varepsilon/2}(\overline{X})\big\}\big|+  \calE\Big(\overline{X};\frac{3\varepsilon}{2}\Big)\le  c_2   K\varepsilon\sqrt{\log(1/\varepsilon)},\label{calEXb}
    \\
    Q_{1-\varepsilon/2}(\overline{X})-Q_{\varepsilon/2}(\overline{X}) \le c_3 K\sqrt{\log(1/\varepsilon)}.\label{QMQb}
\end{gather}
Setting $n_0=\sqrt{n}\log n$ and $\varepsilon=c_4\sqrt{\frac{\log n}{n}}$ for some suitably large constant $c_4$, we reach the claim.
\end{proof}

\subsection{The Adversarial Noise in the One-Dimensional Case}
We now proceed to examine the estimator's robustness properties against an $\eta$-fraction of adversarial noise. Recall our notion of pre-quantization noise in  \Cref{eq:adversarial_noise_pre} and post-quantization noise in \Cref{eq:adversarial_noise_pos}. They capture corruption arising in different stages of the procedure (see, e.g., \cite{dirksen2021non}). In partial quantization,
 the pre-quantization noise (\ref{eq:adversarial_noise_pre})
 is more delicate to handle because it is possible that the $n_0$ samples from $\{\tilde{X}_i\}_{i=1}^n$ for computing the quantiles $\alpha=Y^*_{\varepsilon n_0}$ and $\beta=Y^*_{(1-\varepsilon)n_0}$ are all corrupted. To avoid this, we will introduce randomness in the selection of $\{Y_i\}_{i=1}^{n_0}$. On the other hand, note that we now only use the bits $\{\dot{X}_i:n_0+1\le i\le n\}$ to construct the estimator, and thus the post-quantization noise should be formulated as 
 
\begin{align}
    |\{n_0+1\le i\le n:\widetilde{\dot{X}}_i=-\dot{X}_i\}|\le \eta n.\label{postcor}
\end{align}   

We first  show that our estimator $\hat{\mu}$ is robust to post-quantization noise (\ref{postcor}) occurring during the quantization process. 
Using the $n-n_0$ corrupted bits $\{\tilde{\dot{X}}_i\}_{i=n_0+1}^n$, our estimator is given by 
\begin{align}
\label{eq:robustes}
    \widetilde{\mu}:=\frac{\Delta}{n-n_0}\sum_{i=n_0+1}^n \widetilde{\dot{X}}_i+\mu_1,
\end{align}
where $\mu_1$ is defined in \Cref{def:mu_1}.

\begin{theorem}[Univariate with post-quantization noise in \eqref{postcor}] \label{thm:corr11}
There exist absolute constants $C,c_1,c_2>0$ for which the following holds.
Consider the estimator $\widetilde{\mu}$ in \Cref{eq:robustes} with $\{\widetilde{\dot{X}}_i\}_{i=n_0+1}^n$ obeying (\ref{postcor}), where $\{\widetilde{X}_i\}_{i=1}^{n_0}$ are the same as in Theorem \ref{thm:real_corfree}. For any $\eta\le 1/c_2$, then with probability at least $1-\exp(-c_1n_0\varepsilon)-\frac{2}{n}$, we have
\begin{align}\label{genebound22}
    |\widetilde{\mu}-\mu| &\le \calE\Big(\overline{X};\frac{3\varepsilon}{2}\Big) +3\varepsilon\big|\max\big\{-Q_{1-\varepsilon/2}(\overline{X}),Q_{3\varepsilon/2}(\overline{X})\big\}\big|\\
    &+(Q_{1-\varepsilon/2}(\overline{X})-Q_{\varepsilon/2}(\overline{X}))\left(\sqrt{\frac{\log n}{n}}+2\eta\right).
\end{align}
In particular, if $X$ is $K$-subgaussian then under $n_0=\sqrt{n}\log n$ and $\varepsilon=c_2\max\{\eta,\sqrt{\log n/n}\}$, with probability at least $1-\frac{1}{n}$ we have
\begin{align}
    |\widetilde{\mu}-\mu| \le CK \left(\frac{\log n}{\sqrt{{n}}}+\eta \sqrt{\log(1/\eta)}\right).
\end{align}
\end{theorem}
We remark that the additional error term $\eta\sqrt{\log(1/\eta)}$ matches the minimax lower bound for mean estimation under adversarial noise with subgaussian $X$ from \cite{lugosi2021trimmed}. 

\begin{proof}
Let $\widehat{\mu}$ be estimator in Theorem \ref{thm:real_corfree}. We only need to estimate the additional error term due to the adversarial contamination $|\widetilde{\mu}-\hat{\mu}|$. To this end, by \Cref{postcor}, we have
    \begin{align}
        &|\widetilde{\mu}-\hat{\mu}|=\frac{\Delta}{n-n_0}\left|\sum_{i=n_0+1}^n \big(\widetilde{\dot{X}}_i-\dot{X}_i\big)\right|\le \frac{2\Delta \eta n}{n-n_0}\le 4\Delta\eta. 
    \end{align}
    Moreover, by (\ref{Detabound}) in the proof of Theorem \ref{thm:real_corfree}, with the promised probability, we have 
    \begin{align}
        |\widetilde{\mu}-\hat{\mu}|\le 2\eta \big(Q_{1-\varepsilon/2}(\overline{X})-Q_{\varepsilon/2}(\overline{X})\big). 
    \end{align}
    Combining with (\ref{genebou11}) yields (\ref{genebound22}).  To obtain the bound for subgaussian $X$, we recall the estimates in \Cref{calEXb} and \Cref{QMQb}. Notice that whenever $\eta \le \sqrt{n/\log n}$ the estimates remain the same as in Theorem \ref{thm:real_corfree}. It remains to observe that $\log 1/\varepsilon = O( \log n)$. This is certainly the case when $\eta<\sqrt{n/\log n}$. On the other hand, when $\eta\ge \sqrt{n/\log n} $ it holds that $\varepsilon=c_2\eta $ and that  $\log (1/\eta) \le \log (n)/2$.
\end{proof}

Next, we consider pre-quantization noise (\ref{eq:adversarial_noise_pre}) which occurs before our procedure and corrupts the entire sample $\{X_i\}_{i=1}^n$ to $\{\widetilde{X}_i\}_{i=1}^n$. In this case, our previous construction  may fail to be robust. In fact, our goal is to keep the regime $n_0 =o(n)$, then it might happen that all samples in the set $\calX_{\rm (f)}=\{Y_1,\cdots,Y_{n_0}\}$ are corrupted which compromises entirely our procedure.

To resolve this issue, we propose to choose the indices $\{i_1,i_2,\cdots,i_{n_0}\}$
uniformly at random without replacement from the total of $n$ indexes. Next, we consider the set $\{\widetilde{X}_{i_1},\widetilde{X}_{i_2},\cdots,\widetilde{X}_{i_{n_0}}\}$, which we shall (with a slight abuse of notation) denote by $\{Y_1,Y_2,\cdots,Y_{n_0}\}$. Given the set of $n_0$ unquantized samples, we compute the quantiles 
\begin{align*}
\alpha:=Y^*_{\varepsilon n_0} \quad \text{and} \quad \beta:=Y^*_{(1-\varepsilon)n_0},
\end{align*}
for some well-chosen $\varepsilon>0$. We then follow similarly the previous quantization approach over the remaining $n-n_0$ samples. More accurately, without loss of generality, suppose that $i_1=1,i_2=2,\ldots,i_{n_0}=n_0$ and consider the sample $\{\widetilde{X}_{n_0+1},\cdots,\widetilde{X}_n\}$, we compute 
\begin{align*}
\mu_1:=\frac{\alpha+\beta}{2} \quad \text{and} \quad \Delta :=\frac{
\beta-\alpha
}{2},
\end{align*}
as before and we draw i.i.d. $\{\tau_i\}_{i=n_0+1}^n$ from ${ \scrU}([-1,1])$ to apply the quantization map
\begin{align}
\label{coquanti}
    \dot{\widetilde{X}}_i:= \sign(\widetilde{X}_i-\mu_1+\Delta\tau_i),\qquad n_0+1\le i\le n. 
\end{align}
Our estimator has the same form as before in \Cref{eq:robustes}:
\begin{align}
\label{eq:tildemumu}
    \widetilde{\mu}=\frac{\Delta}{n-n_0}\sum_{i=n_0+1}^n \widetilde{\dot{X}}_i+\mu_1.
\end{align}
We now show that with our sampling strategy, we can obtain similar guarantee as in the previous theorem without any restriction on which samples the adversary may affect. 
\begin{theorem}[Univariate with corruption from \Cref{eq:adversarial_noise_pre}] \label{thm43prequan}
There exist absolute constants $C,c_1,c_2>0$ for which the following holds. 
Assume an arbitrary $\eta$-fraction corruption as in \Cref{eq:adversarial_noise_pre}, for some $0<\eta\le 1/8$. Consider the estimator in \Cref{eq:tildemumu} based on the quantized data from \Cref{coquanti}. Then for any $\varepsilon\ge 8\eta$, with probability at least $1-5\exp(-\frac{\varepsilon n_0}{24})-\frac{2}{n}$ we have
\begin{align*}
    |\widetilde{\mu}-\mu|&\le \calE(\overline{X};2\varepsilon) + 4\varepsilon\bigg|\max\big\{-Q_{1-\varepsilon/2}(\overline{X}),Q_{\varepsilon/2}(\overline{X})\big\}\bigg|\\
    &+  \big(Q_{1-\varepsilon/2}(\overline{X})-Q_{\varepsilon/2}(\overline{X})\big)\left(\frac{3\eta}{2} + \sqrt{\frac{\log n}{n}}\right).
\end{align*}
Specifically, if $X$ is $K$-subgaussian then for $\varepsilon = \max\{8\eta,c_2\sqrt{\log n/n}\}$ and $n_0 =\sqrt{n}\log n$, with probability at least $1-\frac{1}{n}$ we have
\begin{align}\label{sgraterobust}
    |\widetilde{\mu}-\mu|\le CK \left(\frac{\log n}{\sqrt{n}}+\eta \sqrt{\log(1/\eta)}\right).
\end{align}
\end{theorem}
\begin{proof}
Let $Z$ be the number of corrupted samples in the set $\{Y_1,\cdots,Y_{n_0}\}$. Notice that $Z$ follows a hyper-geometric distribution with $n_0$ draws from a population of size $n$, where $\eta n$ are corrupted samples. Notice that $Z$ has mean $\eta n_0$ which is stochastically dominate by another hyper-geometric distribution $Z'$ with $n_0$ draws from a population of size $n$ where $(\varepsilon/6) n_0> \eta n_0$ are corrupted samples. Thus, setting  $p\coloneqq \varepsilon/6 $, by the standard tail estimate for the hyper-geometric distribution (see for example \cite{hoeffding1963probability}) we have for every $0<t<p$,
\begin{equation*}
\mathbb{P}(Z\ge (p+t)n_0)\le \mathbb{P}(Z'\ge (p+t)n_0)\le e^{-D(p-t||p)}.
\end{equation*}
By a simple Taylor series expansion, we know that $D(3p/2||p)\ge 0.1 p$ for $p$ sufficiently small. In particular, setting $t=p/2$ we ensure that
\begin{equation}
\label{eq:event_Z}
    \mathbb{P}\bigg(Z\le \frac{\varepsilon}{4} n_0\bigg) \ge 1-e^{-0.1\varepsilon n_0}.
\end{equation}
For the rest of the proof, we assume that the event above holds. Next, we proceed to estimate $\alpha$ and $\beta$. To this end, we first consider $n_0$ uncorrupted samples $X_1,\cdots,X_{n_0}$.  By Chernoff's inequality, it holds that
\begin{equation*}
\mathbb{P}\left(|\{i\in[n_0]:X_i\ge Q_{1-2\varepsilon}(X)\}| \ge \frac{3}{2}\varepsilon n_0\right)\ge 1-e^{-\varepsilon n_0/24},
\end{equation*}
and that
\begin{equation*}
\mathbb{P}\left(|\{i\in[n_0]:X_i\le Q_{1-\varepsilon/2}(X)\}| \ge \big(1-\frac{3}{4}\varepsilon\big) n_0\right)\ge 1-e^{-\varepsilon n_0/24}.
\end{equation*}
To take into account the adversarial noise, notice by our estimate on $Z$ (see \Cref{eq:event_Z}), we obtain that with probability at least $1-2e^{-\varepsilon n_0/24}$, it holds that
\begin{align}
&|\{i\in[n_0]:Y_i\ge Q_{1-2\varepsilon}(X)\}|\ge \bigg(\frac{3}{2}\varepsilon -\frac{\varepsilon}{4}\bigg)n_0 \ge \varepsilon n_0 \quad \text{and that} \\
&|\{i\in[n_0]:Y_i\le Q_{1-\varepsilon/2}(X)\}|\ge \bigg(1-\frac{3}{4}\varepsilon -\frac{\varepsilon}{4}\bigg)n_0 = (1-\varepsilon)n_0.
\end{align}
In this event, we guarantee that 
\begin{equation}
\label{eq:bound_beta_robust}
   Q_{1-2\varepsilon}(X) \le \beta \le Q_{1-\varepsilon/2}(X).
\end{equation}
Similarly, with probability at least $1-2e^{-\varepsilon n_0/24}$ we have 
\begin{equation}
\label{eq:bound_alpha_robust}
   Q_{\varepsilon/2}(X) \le \alpha \le Q_{2\varepsilon}(X).
\end{equation}
We are now ready to analyze the proposed estimator. In what follows, we assume that the events from \Cref{eq:bound_alpha_robust,eq:bound_beta_robust,eq:event_Z} hold.

The argument now re-iterates the arguments presented in Theorems \ref{thm:real_corfree} and \ref{thm:corr11}. To start, we control the impact of the corruption over $\{X_i\}_{i=n_0+1}^n$. We define $\dot{X}_i=\sign(X_i-\mu_1+\tau_i)$ when $n_0+1\le i\le n$ and consider
$$\widehat{\mu} = \frac{\Delta}{n-n_0}\sum_{i=n_0+1}^n \dot{X}_i +\mu_1.$$
By triangle inequality, we have
\begin{align*}
|\widetilde{\mu}-\mu|\le |\widetilde{\mu}-\hat{\mu}|+|\mathbb{E}\hat{\mu}-\mu| + |\hat{\mu}-\mathbb{E}\hat{\mu}|.
\end{align*}
We estimate each of the three terms in the right-hand side separately. For the first term, we repeat the argument used to derive \Cref{eq:adversarial_noise_pre} and \Cref{postcor},
\begin{align*}
        &|\widetilde{\mu}-\widehat{\mu}|\le \frac{2\Delta \eta n}{n-n_0}\le 3\eta\Delta\le \frac{3\eta}{2}(\beta-\alpha)\\
        &\le \frac{3}{2}\eta\big(Q_{1-\varepsilon/2}(\overline{X})-Q_{\varepsilon/2}(\overline{X})\big) \quad \text{(from \Cref{eq:bound_alpha_robust,eq:bound_beta_robust})}. 
\end{align*}
For the second term, we recall that $\overline{X}=X-\E X$ and repeat the same argument used in the proof of \Cref{deviation}) to obtain that
 \begin{align*}
|\mathbb{E}\widehat{\mu}-\mu|\le \calE(\overline{X};2\varepsilon) + 4\varepsilon\max\big|\big\{Q_{1-\varepsilon/2}(\overline{X}),Q_{\varepsilon/2}(\overline{X})\big\}\big|. 
\end{align*}
It remains to bound the third term $|\widehat{\mu}-\mathbb{E}\hat{\mu}|$. To this end, we use Hoeffding's inequality over the randomness of $\{\tau_i\}_{i=n_0+1}^n$ to obtain 
\begin{align*}
\mathbb{P}\left(|\widehat{\mu}-\mathbb{E}\widehat{\mu}|\le(\beta-\alpha)\sqrt{\frac{\log n}{n}}\right)=\mathbb{P}\left(|\widehat{\mu}-\mathbb{E}\widehat{\mu}|\le2\Delta\sqrt{\frac{\log n}{n}}\right)  \ge 1-\frac{2}{n}.
 \end{align*}
Recalling the estimates for $\alpha$ and $\beta$ derived in \Cref{eq:bound_alpha_robust} and \Cref{eq:bound_beta_robust} concludes the proof of the first part.

To obtain the second part of the statement when $X$ is subgaussian, recall (\ref{calEXb}) and (\ref{QMQb}), for small enough  $\varepsilon$ and some absolute constants $c_5,c_6$, we have 
\begin{gather*}
 \calE(\overline{X};2\varepsilon) + 4\varepsilon\big|\max\big\{-Q_{1-\varepsilon/2}(\overline{X}),Q_{\varepsilon/2}(\overline{X})\big\}\big|\le c_5 K\varepsilon\sqrt{\log(1/\varepsilon)}\,,\\
    \eta|Q_{1-\varepsilon/2}(\overline{X})-Q_{\varepsilon/2}(\overline{X})|\le \varepsilon|Q_{1-\varepsilon/2}(\overline{X})-Q_{\varepsilon/2}(\overline{X})|\le c_6K \varepsilon\sqrt{\log(1/\varepsilon)}\,.
\end{gather*}
Arguing as in the end of the proof of Theorem \ref{thm:corr11}, we obtain the announced convergence rate.
\end{proof}

\subsection{The Multivariate Case}
We now show how our constructions extend to the multivariate case. To start, we separate the sample $X_1,\cdots,X_n$ as before into the sets $\calX_{\rm (f)}=\{Y_1,\cdots,Y_{n_0}\}$ and $\calX_{\rm (q)}=\{X_{n_0+1},\cdots,X_n\}$. We shall compute the quantiles for each $j\in[d]$,
\begin{equation*}
\alpha_j = \langle Y,e_j\rangle^{*}_{\varepsilon n_0} \quad \text{ and} \quad \beta_j=\langle Y,e_j\rangle^*_{(1-\varepsilon)n_0} ,
\end{equation*}
and compute the vectors $\alpha=(\alpha_1,\cdots,\alpha_d)^\top$ and $\beta=(\beta_1,\cdots,\beta_d)^\top$. 

Similarly as before, we compute $\mu_1:=\frac{1}{2}(\alpha+\beta)$ and $\Delta:=\frac{1}{2}(\beta-\alpha)$. Next we quantize the samples, we draw $\tau_{n_0+1},\cdots,\tau_{n}\stackrel{iid}{\sim}{\scrU}[-1,1]^d$ and then quantize $\{X_i\}_{i=n_0+1}^n$ to
\begin{align}
    \dot{X}_i = \sign(X_i-\mu_1 + \diag(\Delta)\tau_i),
\end{align}
The estimator is given by
\begin{align}\label{multi1besti}
    \widehat{\mu} = \frac{1}{n-n_0}\sum_{i=n_0+1}^{n}\diag(\Delta)\dot{X}_i+\mu_1,
\end{align}
which amounts to implementing the univariate procedure for each entry. Consequently, our estimator achieves a reduction of the number of bits from $32nd$ to $32n_0d+ (n-n_0)d=(1+o(1))nd$. Meanwhile, the next theorem shows that our estimator enjoys non-asymptotic guarantees that are only slightly larger (up to a log factor) than the existing estimates for estimators based on unquantized samples.

\begin{theorem}[Multivariate without corruption] 
\label{thm:partial_multivariate_no_corruption}
There exist  absolute constants $C,c_1,c_2,c_3>0$ for which the following holds.  Consider the above quantization and estimation procedure with parameters $n_0\le \frac{n}{2}$ and $\varepsilon>0$. With probability at least $1-4d\exp(-c_1\varepsilon n_0)-\frac{2}{d}$, we have
\begin{align*}
        \|\hat{\mu}-\mu\|_2&\le c_2\left(\sum_{j=1}^d \calE(\overline{X}_j;2\varepsilon)^2+\varepsilon^2\max\big\{-Q_{1-\varepsilon/2}(\langle \overline{X},e_j\rangle),Q_{3\varepsilon/2}(\langle \overline{X},e_j\rangle)\big\}^2\right)^{1/2}\\
        &+c_2\sqrt{\frac{\sum_{j=1}^d[Q_{1-\varepsilon/2}(\langle X_i,e_j\rangle)-Q_{\varepsilon/2}(\langle X_i,e_j\rangle)]^2}{n}}\\
        &+c_2\sqrt{\frac{\log d\max_{j\in[d]}[Q_{1-\varepsilon/2}(\langle X_i,e_j\rangle)-Q_{\varepsilon/2}(\langle X_i,e_j\rangle)]^2}{n}}.
    \end{align*}
    Specifically, for $n_0=\sqrt{n}\log d$ and $\varepsilon=\frac{c_{3}}{\sqrt{n}}$, if $X$ is $K$-subgaussian then with probability at least $1-\frac{1}{d}$,
    \begin{align}\label{hatmul2sg}
        \|\hat{\mu}-\mu\|_2 \le CK \log(n) \sqrt{\frac{\tr(\Sigma)+\|\Sigma\|\log(d)}{n}}.
    \end{align}
    \end{theorem}
One can see that for unquantized estimators with subgaussian data, the best one can hope for is $O(\sqrt{\tr(\Sigma)/n + \|\Sigma\|\log d/n})$ once we choose the confidence level to be $\delta \asymp d^{-1}$. Thus our estimate is optimal up to a mild $\log(n)$ factor.

\begin{proof}
We apply the triangle inequality to split the error estimate into bias and concentration terms as before:
\begin{align}\label{hdtri}
    \|\hat{\mu}-\mu\|_2\le \|\mathbb{E}\hat{\mu}-\mu\|_2 + \|\hat{\mu}-\mathbb{E}\hat{\mu}\|_2. 
\end{align}
To control the bias, we fix an index $j\in[d]$ and apply the Chernoff bound to re-iterate the arguments in the proof of Theorem \ref{thm:real_corfree}. In particular, we obtain that
    \begin{gather}
        \label{alphajbound}Q_{\varepsilon/2}(\langle X_i,e_j\rangle)\le\alpha_j\le Q_{3\varepsilon/2}(\langle X_i,e_j\rangle),\\\label{betajbound}
        Q_{1-3\varepsilon/2}(\langle X_i,e_j\rangle)\le \beta_j\le Q_{1-\varepsilon/2}(\langle X_i,e_j\rangle)
    \end{gather}
with probability at least $1-4\exp(-c_1\varepsilon n_0)$. Taking a union bound over all indexes $j\in [d]$, we conclude that the events in \Cref{alphajbound,betajbound} are likely, that is, with probability at least $1-4d\exp(-c_1\varepsilon n_0)$ all the events hold simultaneously for all indexes. Consequently, we have
    \begin{align}
        \label{Deltajbb}
        \Delta_j \le \frac{1}{2}\big[Q_{1-\varepsilon/2}(\langle X_i,e_j\rangle)-Q_{\varepsilon/2}(\langle X_i,e_j\rangle)\big].
    \end{align}
Next, we can re-iterate the arguments used to derive \Cref{deviation} to obtain that for every $j\in [d]$,
\begin{align*}
        |\mathbb{E}(\hat{\mu}_j)-\mu_j|\le\calE(\langle \overline{X},e_j\rangle;3\varepsilon/2) + 3\varepsilon\max\big|\big\{Q_{1-\varepsilon/2}(\langle \overline{X},e_j\rangle),Q_{3\varepsilon/2}(\langle \overline{X},e_j\rangle)\big\}\big|,
\end{align*}
which implies
\begin{align}\label{hddevi}
        \|\mathbb{E}\hat{\mu}-\mu\|_2 \le \left(\sum_{j=1}^d 2\calE(\overline{X}_j;2\varepsilon)^2+18\varepsilon^2|\big\{Q_{1-\varepsilon/2}(\langle \overline{X},e_j\rangle),Q_{3\varepsilon/2}(\langle \overline{X},e_j\rangle)\big\}\big|^2\right)^{1/2}.  
\end{align}
    It remains to control the concentration term $\|\hat{\mu}-\mathbb{E}\hat{\mu}\|_2$ which can be written as
    \begin{align}\nn
        &\|\hat{\mu}-\mathbb{E}\hat{\mu}\|_2 = \left\|\frac{1}{n-n_0}\sum_{i=n_0+1}^n\diag(\Delta)\big[\dot{X}_i
        -\mathbb{E}\dot{X}_i \big]\right\|_2 \\
        &= \sup_{v\in \mathbb{S}^{d-1}}\left\langle \diag(\Delta)v,\frac{1}{n-n_0}\sum_{i=n_0+1}^n\big[\dot{X}_i-\mathbb{E}\dot{X}_i\big]\right\rangle\\
        &= \sup_{v\in \diag(\Delta)\mathbb{S}^{d-1}}\left\langle v,\frac{1}{n-n_0}\sum_{i=n_0+1}^n\big[\dot{X}_i-\mathbb{E}\dot{X}_i\big]\right\rangle.\nn
    \end{align}
    We condition on $\{X_i\}_{i=n_0+1}^n$, $\alpha$, $\beta$ and utilize the randomness of $\{\tau_i\}_{i=n_0+1}^n\stackrel{iid}{\sim}{ \scrU}[-1,1]^d$, then the entries of $\dot{X}_i=\sign(X_i-\mu_1+\diag(\Delta)\tau_i)$ are independent. Therefore, for any $v_1,v_2\in \diag(\Delta)\mathbb{S}^{d-1}$, we have
    \begin{align*}
        &\left\|\left\langle v_1-v_2,\frac{1}{n-n_0}\sum_{i=n_0+1}^n\big[\dot{X}_i-\mathbb{E}\dot{X}_i\big]\right\rangle\right\|_{\psi_2}^2 
        \\
        & \le \frac{c_2}{(n-n_0)^2}\sum_{i=n_0+1}^n\big\|\langle v_1-v_2,\dot{X}_i-\mathbb{E}\dot{X}_i\rangle\big\|_{\psi_2}^2 \le \frac{c_3\|v_1-v_2\|_2^2}{n-n_0}.
    \end{align*}
   By the standard estimate for the empirical mean, we have
    \begin{align}
    \label{eq:estimate_concentration_last}
        &\|\hat{\mu}-\mathbb{E}\hat{\mu}\|_2 \le c_5\sqrt{\frac{\sum_{j=1}^d\Delta_j^2 + t^2 \max_{j\in[d]}\Delta_j^2}{n-n_0}}\\
        & \stackrel{(\ref{Deltajbb})}{\le} \sqrt{\frac{\sum_{j=1}^d[Q_{1-\varepsilon/2}(\langle X_i,e_j\rangle)-Q_{\varepsilon/2}(\langle X_i,e_j\rangle)]^2}{n}}\\
        &~~+ t\frac{\max_{j\in[d]}[Q_{1-\varepsilon/2}(\langle X_i,e_j\rangle)-Q_{\varepsilon/2}(\langle X_i,e_j\rangle)]}{\sqrt{n}},
    \end{align}
    with probability at least $1-2\exp(-t^2)$. We set $t=\sqrt{\log d}$ to obtain that
    \begin{align}
    \label{hdconcen}
        &\|\hat{\mu}-\mathbb{E}\hat{\mu}\|_2\\
        &\le c_6 \sqrt{\frac{\sum_{j=1}^d[Q_{1-\varepsilon/2}(\langle X_i,e_j\rangle)-Q_{\varepsilon/2}(\langle X_i,e_j\rangle)]^2}{n}}\\
        &+ \max_{j\in[d]}[Q_{1-\varepsilon/2}(\langle X_i,e_j\rangle)-Q_{\varepsilon/2}(\langle X_i,e_j\rangle)]\sqrt{\frac{\log d}{n}}
    \end{align}
    holds with probability at least $1-\frac{2}{d}$. Substituting (\ref{hddevi}) and (\ref{hdconcen}) into (\ref{hdtri}) yields the claim.

    The subgaussian case is somewhat analogous to (\ref{calEXb})--(\ref{QMQb}). Indeed, let $\Sigma_{jj}$ to be the $j$-th diagonal entry of the covariance matrix of $X$  and notice that
    \begin{align*}
        &\varepsilon\max\big|\big\{-Q_{1-\varepsilon/2}(\langle \overline{X},e_j\rangle),Q_{3\varepsilon/2}(\langle \overline{X},e_j\rangle)\big\}\big|+\calE(\langle \overline{X},e_j\rangle;2\varepsilon) \\
        &\le c_7  \varepsilon\sqrt{\Sigma_{jj}\log(1/\varepsilon)}\\
        &\text{and that,}\\
        &Q_{1-\varepsilon/2}(\overline{X}_j)-Q_{\varepsilon/2}(\overline{X}_j)\le c_8  \sqrt{\Sigma_{jj}\log(1/\varepsilon)}.
    \end{align*}
Substituting them into \Cref{eq:estimate_concentration_last}, we have
    \begin{align*}
        \|\hat{\mu}-\mu\|_2\le c_9\left(\varepsilon\sqrt{\log(1/\varepsilon)\tr(\Sigma)}+\sqrt{\frac{\log(1/\varepsilon)\tr(\Sigma)+\log(d)\log(1/\varepsilon)\|\Sigma\|}{n}} \right)
    \end{align*}
    with probability at least $1-4d\exp(-c_1\varepsilon n_0)-\frac{2}{d}$. Now  
    setting $n_0=\sqrt{n}\log d$ and $\varepsilon=\frac{c_{11}}{\sqrt{n}}$ for some large enough $c_{11}$, we obtain
    that 
    \begin{align}
        \|\hat{\mu}-\mu\|_2 \le c_{12}\sqrt{\frac{\log(n)\tr(\Sigma)+\|\Sigma\|\log(n)\log(d)}{n}} 
    \end{align}
    holds with probability at least $1-\frac{6}{d}$. By re-scaling $d$ and the absolute constants, we finish the proof.
\end{proof}

As already pointed at the end of Section \ref{sec:fullquan},  empirical mean estimators do not seem to enjoy fast convergence rate when the data is contaminated with adversarial noise. To bypass this issue, we combine the estimator proposed by Depersin and Lecue with the sampling strategy in the one-dimensional case to avoid too many corrupted samples among the $n_0$ samples we use to compute the empirical quantiles. 

Recall that we choose $n_0$ samples $Y_1,\ldots,Y_{n_0}$  uniformly at random without replacement among a total of $n$ samples $X_1,\ldots,X_n$. Next, we repeat the same procedure as in Theorem \ref{thm43prequan},
\begin{equation*}
\alpha_j = \langle Y,e_j\rangle^{*}_{\varepsilon n_0} \quad \text{ and} \quad \beta_j=\langle Y,e_j\rangle^*_{(1-\varepsilon)n_0} ,
\end{equation*}
and compute the quantities $\alpha=(\alpha_1,\cdots,\alpha_d)^\top$, $\beta=(\beta_1,\cdots,\beta_d)^\top$, $\mu_1$ and $\Delta$ as before. The quantization scheme also remains the same
\begin{align}
    \dot{X}_i = \sign(X_i-\mu_1 + \diag(\Delta)\tau_i),
\end{align}
however the estimator changes. Let $\widehat{\mu}_{DL}$ be the estimator from Proposition \ref{prop:depersin_lecue} applied to the samples $\diag(\Delta)\dot{X}_1,\cdots,\diag(\Delta)\dot{X}_n$ and consider the estimator given by
\begin{align}
    \widehat{\mu} = \widehat{\mu}_{DL} +\mu_1.
\end{align}

\begin{theorem}[Multivariate with corruption] 

\label{thm:partial_multivariate_corruption}
There exist absolute constants $C,c_1,c_2,c_3>0$ for which the following holds. Assume that an $\eta$-fraction of the sample was corrupted for some $\eta < 1/8$.  Consider the above quantization and estimation procedure with parameters $n_0\le \frac{n}{2}$ and $\varepsilon>0$. Then, for $\varepsilon\ge 8\eta$, with probability at least $1-d\exp(-c_1\varepsilon n_0)-\frac{1}{d}$, we have
\begin{align*}
        \|\hat{\mu}-\mu\|_2&\le c_2\left(\sum_{j=1}^d \calE(\overline{X}_j;2\varepsilon)^2+\varepsilon^2\max\big\{-Q_{1-\varepsilon/2}(\langle \overline{X},e_j\rangle),Q_{3\varepsilon/2}(\langle \overline{X},e_j\rangle)\big\}^2\right)^{1/2}\\
        &+c_2\sqrt{\frac{\sum_{j=1}^d[Q_{1-\varepsilon/2}(\langle X_i,e_j\rangle)-Q_{\varepsilon/2}(\langle X_i,e_j\rangle)]^2}{n}}\\
        &+c_2\max_{j\in[d]}[Q_{1-\varepsilon/2}(\langle X_i,e_j\rangle)-Q_{\varepsilon/2}(\langle X_i,e_j\rangle)]\sqrt{\frac{\log d}{n}+\eta}.
    \end{align*}
     Specifically, for $n_0=\sqrt{n}\log d$ and $\varepsilon=\max\{8\eta,\frac{c_{3}}{\sqrt{n}}\}$, if $X$ is $K$- subgaussian then with probability at least $1-\frac{1}{d}$,
    \begin{align}
        \|\hat{\mu}-\mu\|_2 \le CK \log(n) \sqrt{\frac{\tr(\Sigma)+\|\Sigma\|\log(d)}{n}+\|\Sigma\|\eta}.
    \end{align}
\end{theorem}
\begin{proof}
We apply the same argument used to derive estimates analogous to \eqref{eq:bound_alpha_robust} and \eqref{eq:bound_beta_robust} for $\langle X,e_1\rangle,\ldots,\langle X,e_d\rangle$. Thus, by a simple union bound, we obtain that with probability at least $1-4de^{-\varepsilon n_0}-de^{-0.1\varepsilon n_0}$, for every $i\in [d]$,
\begin{align*}
&Q_{1-2\varepsilon}(\langle X,e_i\rangle) \le \beta_i \le Q_{1-\varepsilon/2}(\langle X,e_i\rangle).\\
&   Q_{\varepsilon/2}(\langle X,e_i\rangle) \le \alpha_i \le Q_{2\varepsilon}(\langle X,e_i\rangle)\\
& Q_{\varepsilon/2}(\langle X,e_i\rangle)\le \Delta_i \le Q_{1-\varepsilon/2}(\langle X,e_j\rangle).
\end{align*}
For the rest of the proof, we condition on the event where the estimates above hold. By triangle inequality
\begin{align*}
\|\widehat{\mu}-\mu\|_2&\le \|\widehat{\mu}-\mu_1-\E \diag(\Delta)\dot{X}\| + \|\mu - \E  \diag(\Delta)\dot{X}+\mu_1\|_2\\
&= \|\widehat{\mu}_{DL}-\E \diag(\Delta)\dot{X}\| + \|\mu - (\E\diag(\Delta)\dot{X}+\mu_1)\|_2,
\end{align*}
where the second term can be estimated using the same argument as in the proof of Theorem \ref{thm:partial_multivariate_no_corruption}. To handle the first term on the right-hand side, the first observation is to notice that $\diag(\Delta) \dot{X}$ is a random vector whose covariance matrix is a diagonal matrix with entries $\Delta_1^2,\ldots,\Delta_d^2$. The second observation is that the mean estimator $\widehat{\mu}_{DL}$ has an input consisting on $n-n_0$ samples, where at most $\eta n$ where contaminated. Thus, under the regime $n_0=o(n)$, no more than $2\eta(n-n_0)$ samples are corrupted, provided that $n$ is sufficiently large. Therefore, by Proposition \ref{prop:depersin_lecue} for the choice of confidence interval $\delta \asymp d^{-1}$, we obtain that with probability $1-d^{-1}$
\begin{align*}
\|\widehat{\mu}_{DL}-\E \diag(\Delta) \dot{X}\|_2\le C \left(\max_{j\in [d]}\Delta_j\sqrt{\frac{\log(d)}{n}+\eta} + \sqrt{\frac{\sum_{j=1}^d \Delta_j^2}{n}}\right),
\end{align*}
for some absolute constant $C>0$. 
The proof now follows by recalling the estimates $Q_{\varepsilon/2}(\langle X,e_i\rangle)\le \Delta_i \le Q_{1-\varepsilon/2}(\langle X,e_j\rangle)$ valid for all indexes. The second part just applies the same estimates for subgaussian distributions used at the end of Theorem \ref{thm:partial_multivariate_no_corruption}.
\end{proof}

\section{Numerical Examples}
\label{sec:experiments}
In this section, we present numerical experiments that support our theoretical findings. We focus on the estimators developed for the partial quantization regime in Section \ref{sec:partialquan}, as they are the most practically relevant. The only exception is we consider the distributed setting  in Figure \ref{fig5:robustmul}, where the goal is to provide a negative statement. 
Each data point of our results is obtained by averaging over $100$ independent trials.
The codes for reproducing the simulations are available here:
$$\href{https://github.com/junrenchen58/one-bit-mean}{\texttt{https://github.com/junrenchen58/one-bit-mean}}$$

We start with the noiseless univariate setting in Theorem \ref{thm:real_corfree}. Specifically, we test the mean estimation from $n$ i.i.d. $\calN(100,1)$ samples. In our estimator, we roughly follow the parameters in the theorem to choose $n_0=\lceil\sqrt{n}\rceil$ and $\varepsilon=\sqrt{2/n}$. When there is no quantization, we consider the empirical mean $\bar{\mu}=\frac{1}{n}\sum_{i=1}^nX_i$ as the benchmark for the performance of our proposed estimator. We plot the errors obtained when we use the proposed (quantized) estimator $\hat{\mu}$ in Equation (\ref{hatmu}) and when we use the optimal estimator $\bar{\mu}$, for different sample sizes $n\in\{50,100,200,300,400,500\}$ in Figure \ref{fig1:1d}. Remarkably, under the same sample size, the errors obtained with the quantized (one-bit) mean estimator $\hat{\mu}$ are empirically no greater than twice the error obtained using the optimal (unquantized) estimator $\bar{\mu}$, that is, we observe empirically that $|\hat{\mu}-\mu|\le 2|\bar{\mu}-\mu|$. In addition, we  note that our estimator performs well when the mean is $\mu=100$, which is much larger than the variance $1$, confirming that the methods proposed in Section \ref{sec:partialquan} do not depend on the location of the mean. 

\begin{figure}[ht!]
    \centering
    \includegraphics[width=0.65\textwidth]{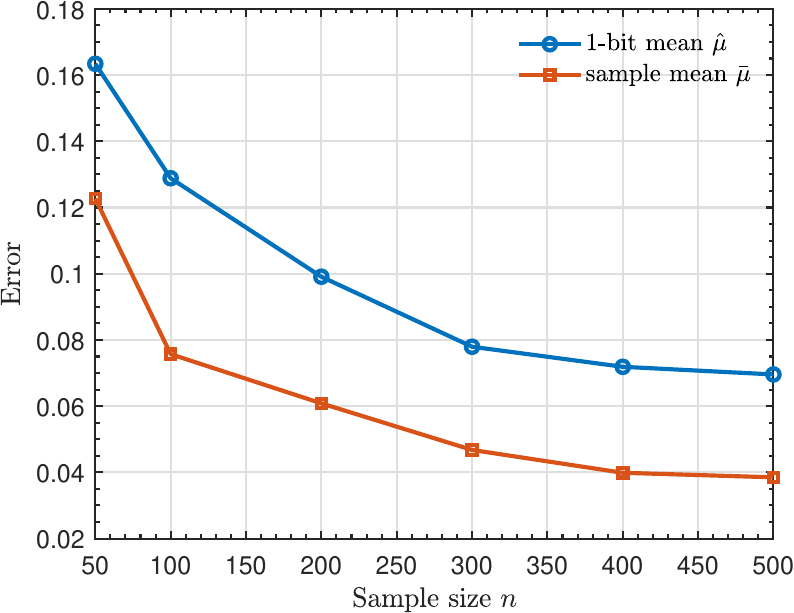}
    \caption{Corruption-free univariate mean estimation.}
    \label{fig1:1d}
\end{figure}

Next, we proceed to the multivariate corruption-free setting in Theorem \ref{thm:partial_multivariate_no_corruption}. We test the estimation of the mean of i.i.d. $\calN(100\cdot \mathbf{1}_d,\Sigma_{\rm Toe})$ samples, where $\Sigma_{\rm Toe}$ is a Toeplitz covariance matrix with $(i,j)$-th entry being $\frac{1}{2^{|i-j|}}$. We test $d=30$ and $n\in \{1200,1500,1800,2100,2400\}$. Similarly, we compare the one-bit estimator $\hat{\mu}$ in Equation (\ref{multi1besti}) with parameters $n_0=\lceil 2\sqrt{n}\rceil$ and $\varepsilon=\frac{3}{\sqrt{n}}$ to the empirical mean $\bar{\mu}=\frac{1}{n}\sum_{i=1}^nX_i$. The results are reported in Figure \ref{fig2:hd}.

\begin{figure}[ht!]
    \centering
    \includegraphics[width=0.65\textwidth]{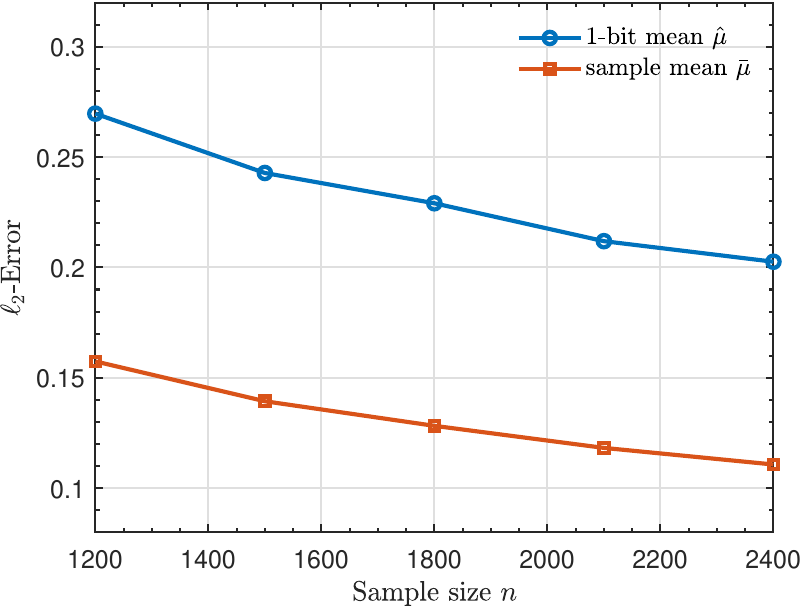}
    \caption{Corruption-free multivariate mean estimation under isotropic covariance.}
    \label{fig2:hd}
\end{figure}

A more challenging regime addressed by our results is the high-dimensional regime, when the dimension $d$ is comparable to the sample size $n$, but the covariance matrix $\Sigma$ presents a low-rank feature, for example, its trace increases slowly with the dimension. In particular, the convergence rate derived in Equation (\ref{hatmul2sg}) scales with the trace of the covariance matrix $\tr(\Sigma)$ up to a logarithmically dependence on the dimension $d$. To capture this regime, we test the mean estimation problem when $X\sim \calN(100\cdot \mathbf{1}_d,\Sigma_d)$ for $d=n/10$ and sample sizes $n\in\{200,400,600,800,1000\}$. The ambient dimension $d$ and the sample size $n$ increase simultaneously, and the ratio $\frac{d}{n}$ remains constant. For a fixed value of $d$, we generate a covariance matrix $\Sigma_d$ by sampling from
$$\Sigma_d = O^T \diag([1^{-2},2^{-2},3^{-2},\cdots,d^{-2}]) O,$$
where $O$ is sampled from the Haar measure of the group of orthogonal transformations. 

Consequently, $\tr(\Sigma_d) = \sum_{i=1}^d i^{-2}$ increases with $d$ slowly; in our setting, $\tr(\Sigma_{d=20})=1.5962$ and $\tr(\Sigma_{d=100})=1.6350$ are very close. The results shown in Figure \ref{fig3:hd} demonstrate the (almost) dimension-free behavior of our one-bit estimator, in particular, the errors decrease with $n$ in spite of the simultaneously increasing $d$. Notably, as in the previous two simulations, the errors of our one-bit estimator typically do not exceed twice of the errors of the sample mean. 

\begin{figure}[ht!]
    \centering
    \includegraphics[width=0.65\textwidth]{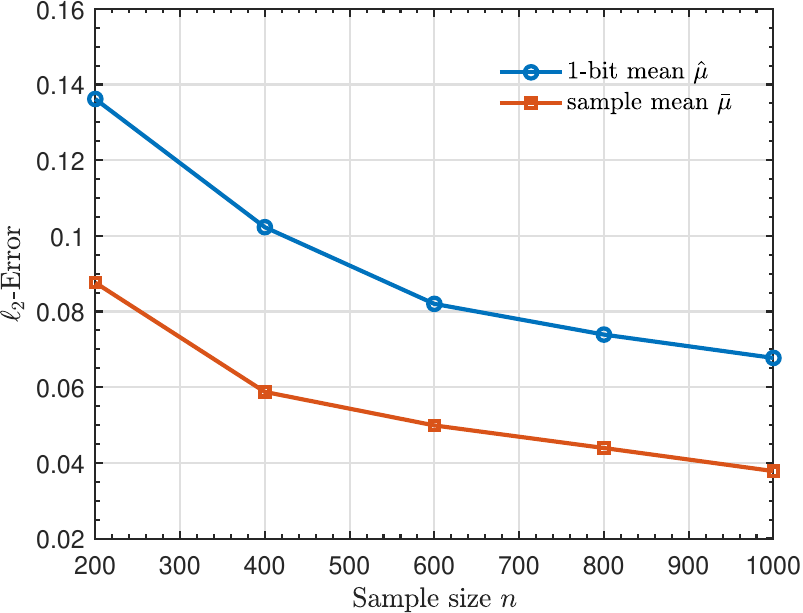}
    \caption{Multivariate mean estimation with low trace $\tr(\Sigma)$.}
    \label{fig3:hd}
\end{figure}

Another important feature of our method is the robustness to an adversarial corruption, which was not considered in prior works of mean estimation under quantization. To illustrate it empirically, we consider the more delicate setting of Theorem \ref{thm43prequan}, where the $\lfloor\eta n\rfloor$ adversarial corruptions occur before drawing $\calX_{\rm (f)}$.  We simulate mean estimation from $n=1000$ samples following $\calN(\mu=100,1)$, while specifically test the corruption that changes the largest $\eta n$ samples in $\{X_i\}_{i=1}^n$ to their symmetric values with respect to $\mu$, that is, changing $X_i$ to $2\mu-X_i$. Note that the sample size is relatively large in univariate mean estimation and therefore the noise term $O(\eta\sqrt{\log(1/\eta)})$ in the error rate in Equation (\ref{sgraterobust}) should dominate. In our estimator $\hat{\mu}$, we set $n_0=\lceil\sqrt{n_0}\rceil$ and draw the size-$n_0$ $\calX_{\rm (f)}$ uniformly without replacement. We also set $\varepsilon=\eta + 1/\sqrt{n}$ and compute $\varepsilon$ and $1-\varepsilon$ quantiles of $\calX_{\rm (f)}$. We shall compare our estimator with the (unquantized) trimmed mean \cite{lugosi2021trimmed} which computes the  $\varepsilon'$ and $1-\varepsilon'$ quantiles of $\{X_i\}_{i=1}^{\xi n}$, denoted respectively by $\alpha'$ and $\beta'$, and then estimates the mean via the remaining $(1-\xi)n$ samples $\{X_i\}_{i= \xi n + 1}^n$ as 
\begin{align}
    \bar{\mu}_{\rm trim}=\frac{1}{(1-\xi)n} \sum_{i=\xi n + 1}^n \phi_{\alpha'}^{\beta'}(X_i).\label{trimmedmeanesi}
\end{align}
We set $\varepsilon'=\eta + 1/\sqrt{n}$ as in our one-bit estimator. In the original paper \cite{lugosi2021trimmed}, the authors set $\xi=1/2$ and thus only use half of the samples to construct the estimator, yet in practice this may affect the convergence rate too much. In our experiment, we test trimmed mean with $\xi\in\{0.1,0.2,0.3,0.4,0.5\}$ and record the smallest error in each independent trial. We show the error curves under $\eta = 0.005:0.005:0.2$ in Figure \ref{fig4:robustuni}. Remarkably, the errors of our $\hat{\mu}$ are very close to $\bar{\mu}_{\rm trim}$. 

\begin{figure}[ht!]
    \centering
    \includegraphics[width=0.65\textwidth]{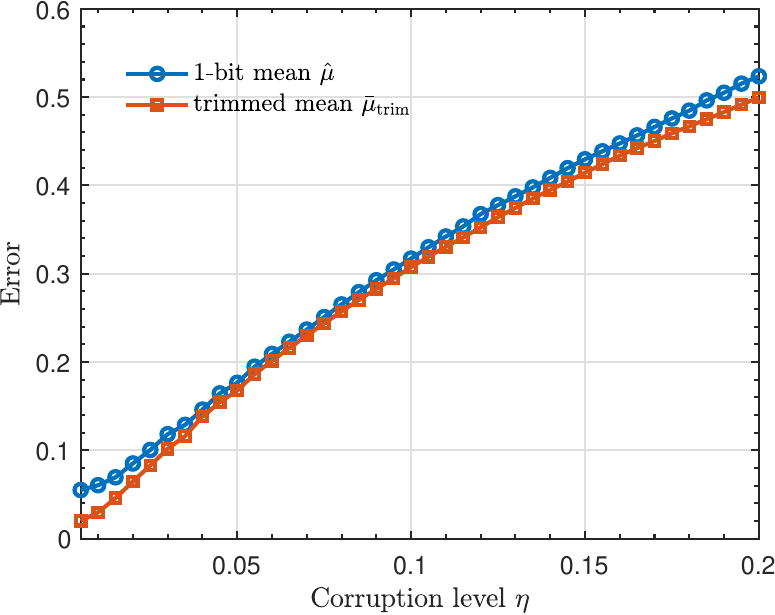}
    \caption{Robust univariate mean estimation.}
    \label{fig4:robustuni}
\end{figure}

The multivariate case under adversarial corruption is much more delicate. Indeed, the Theorems \ref{thm:onebit_ddim_with_robust} and \ref{thm:partial_multivariate_corruption} rely on the estimator proposed in \cite{depersin2022robust} to achieve a convergence rate of order $O(\sqrt{\|\Sigma\|\eta})$. Despite the algorithm proposed in their work enjoys a polynomial-time running time, it does not have an efficient code in practice. It might be the case that the constants in the theorem are too large for practical applications. We leave it as an important open problem.

To conclude, we provide experiments to illustrate that the natural choice of taking the empirical mean of the quantized samples does not enjoy fast rates under an adversarial contamination. We consider the fully quantized setting in Theorem \ref{thm:onebit_ddim_without_robust} for simplicity. We generate $n$ i.i.d. samples from $\calN(0,I_d)$ and set the dithering levels $\lambda_1=\cdots=\lambda_d=2$. We fix the ratio $\frac{d}{n}=\frac{1}{100}$ by testing $d=10:10:100$ under $n=100\cdot d$. Notice that the error rate in our result for the corruption-free regime scales as $\tilde{O}(\sqrt{d/n})$, and therefore the errors should roughly remain constant without corruption. Under corruption, as remarked after Theorem \ref{thm:onebit_ddim_without_robust}, a trivial argument controls the error increment as $O(\eta\sqrt{d})$. In our experiment, we test $\eta=0.05$ and $0.10$ under the corruption pattern that   uniformly selects $\eta n$ samples $X_i$'s and change them to $\tilde{X}_i:=X_i+\mathbf{1}_d$. The ``error v.s. dimension'' curves in Figure \ref{fig5:robustmul} confirm that the errors of the simple estimator in Theorem \ref{thm:onebit_ddim_without_robust} increase with $d$.

\begin{figure}[ht!]
    \centering
    \includegraphics[width=0.65\textwidth]{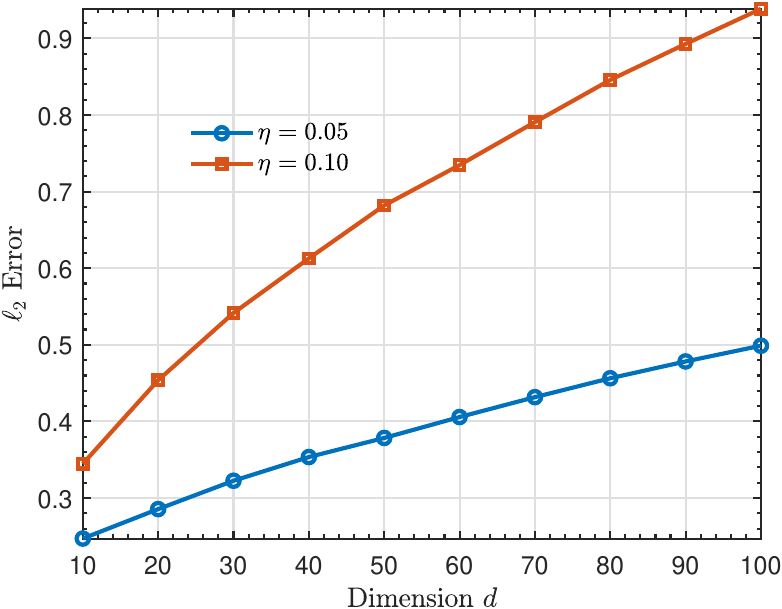}
    \caption{Empirical mean does not enjoy strong robustness.}
    \label{fig5:robustmul}
\end{figure}

\newpage
\subsubsection*{Acknowledgment}
P.A. is supported by NSF and Simons Collaboration on Theoretical Foundations of Deep Learning. J.C. is supported by a Novikov Postdoctoral Fellowship of the Department of Mathematics at the University of Maryland. 
Part of this work was done while J.C. was affiliated with The University of Hong Kong and visiting University of California, Irvine. J.C. is thankful to R. Vershynin and Y. He for the hospitality and stimulating discussions.

\end{document}